\documentclass[12pt]{article}

\usepackage{hyperref}       % hyperlinks
\usepackage{url}            % simple URL typesetting
\usepackage{booktabs}       % professional-quality tables
\usepackage{amsfonts}       % blackboard math symbols
\usepackage{nicefrac}       % compact symbols for 1/2, etc.
\usepackage{microtype}      % microtypography
\usepackage{xcolor}         % colors
\usepackage{algorithmic}
\usepackage[ruled,vlined]{algorithm2e}
\usepackage{amsthm}
\usepackage{bbm}
\usepackage{graphicx}
\usepackage{subfigure}
\usepackage{float}

\usepackage{bbm}
\usepackage[margin=1in]{geometry}
\usepackage{amsmath}
\usepackage{thmtools}
\usepackage{lipsum}

\DeclareMathOperator{\1}{\mathbbm{1}}
\DeclareMathOperator{\Ex}{\mathbb{E}}
\DeclareMathOperator{\Prob}{\mathbb{P}}
\DeclareMathOperator{\F}{\mathcal{F}}
\DeclareMathOperator{\Hp}{\mathcal{H}}

\DeclareMathOperator*{\argmax}{arg\,max}

\newtheorem{theorem}{Theorem}
\newtheorem{lemma}[theorem]{Lemma}
\newtheorem{proposition}[theorem]{Proposition}

\theoremstyle{remark}

\declaretheoremstyle[
spaceabove=\topsep, spacebelow=\topsep,
headfont=\normalfont\bfseries,
notefont=\bfseries, notebraces={}{},
bodyfont=\normalfont\itshape,
postheadspace=0.5em,
name={\ignorespaces},
numbered=no,
headpunct=.]
{mystyle}
\declaretheorem[style=mystyle]{named}
\usepackage[numbers]{natbib}
\title{Batched Thompson Sampling}

\author{%
	Cem Kalkanl\text{\i} and Ayfer \"{O}zg\"{u}r \\
	Department of Electrical Engineering\\
	Stanford University\\
	\texttt{\{cemk, aozgur\}@stanford.edu} \\
}
\date{}
\begin{document}
	
	\maketitle
	
	\begin{abstract}
		We introduce a novel anytime Batched Thompson sampling policy for multi-armed bandits where the agent observes the rewards of her actions and adjusts her policy only at the end of a small number of batches. We show that this policy simultaneously achieves a problem dependent regret of order $O(\log(T))$ and a minimax regret of order $O(\sqrt{T\log(T)})$  while the number of batches can be bounded by $O(\log(T))$ independent of the problem instance over a time horizon $T$. We also show that in expectation the number of batches used by our policy can be bounded by an instance dependent bound of order $O(\log\log(T))$.  These results indicate that Thompson sampling maintains the same performance in this batched setting as in the case when instantaneous feedback is available after each action, while requiring minimal feedback. These results also indicate that Thompson sampling performs competitively with recently proposed algorithms tailored for the batched setting. These algorithms
		optimize the batch structure for a given  time horizon $T$ and prioritize exploration in the beginning of the experiment to eliminate suboptimal actions. We show that Thompson sampling combined with an adaptive batching strategy can achieve a similar performance without knowing the time horizon $T$ of the problem and without having to carefully optimize the batch structure to achieve a target regret bound (i.e. problem dependent vs minimax regret) for a given $T$.

	\end{abstract}
	\section{Introduction}
	
		The multi-armed bandit problem models the scenario where an agent plays repeated actions and observes rewards associated with her actions. The agent aims to accumulate as much reward as possible, and consequently, she has to balance between playing arms that generated high rewards in the past, i.e. exploitation, and selecting under-explored arms that could potentially return better rewards, i.e. exploration. In the ideal scenario, the agent can adjust her policy once she receives feedback, e.g a reward, before the next action instance; however, many real-world applications limit the number of times the agent can interact with the system. For example, in medical applications \cite{thompson1933likelihood}, many patients are treated simultaneously in each trial, and the experimenter has to wait for the outcome of one before planing the next trial. In online marketing \cite{schwartz2017customer}, there may be millions of responses per second, and as a result, it is not feasible for the advertiser to update her algorithm every time she receives feedback. 
	
	Recently, \citet{perchet2016batched} proposed to model this problem as the batched multi-armed bandits. Here, the experiment of duration $T$ is split into a small number of batches and the agent does not receive any feedback regarding the rewards of its actions until the end of each batch. For the two-armed bandit problem, they proposed a general class of batched algorithms called explore-then-commit (ETC) policies, where the agent plays both arms the same number of times until the terminal batch and commits to the better performing arm in the last round unless the sample mean of one arm sufficiently dominates the other in earlier batches. They show that this algorithm achieves the optimal problem-dependent regret $O(\log(T))$ and the optimal minimax regret $O(\sqrt{T})$ matching the performance in the classical case where the agent receives instantaneous feedback about her actions, by using only $\Omega(\log(T/\log(T)))$ and $\Omega(\log\log(T))$ batches respectively. Their algorithm takes the time horizon $T$ and divides it into a fixed number of batches before the experiment where the specific batch structure is tuned to the target performance criteria, i.e. problem-dependent or minimax performance. \citet{gao2019batched} later generalized this result to the setting where the agent had more than two arms to choose from and she could adaptively adjust the batch sizes based on the past data. Their algorithm, called BaSE, is similar to the ETC algorithm in that in each batch the agent plays each of the actions in a set of remaining actions in a round robin fashion, and eliminates the underperforming arms at the end of each batch. They showed that this algorithm required $O(\log(T))$ and $O(\log\log(T))$ number of batches to achieve optimal problem-dependent and minimax regret with batching strategies  specific to each objective. More recently several other batched algorithms appeared in the context of asymptotic optimality \cite{jin2020double}, stochastic and adversarial bandits \cite{esfandiari2019regret}, and linear contextual bandits \cite{han2020sequential,ruan2020linear,ren2020dynamic,ren2020batched}, where the authors provided optimal algorithms in their respective settings. We note that there are also some earlier algorithms developed in the context of the classical bandit setting or bandits with switching cost that even-though not specially developed for the batched setting can be applied in a batched fashion \cite{auer2002finite,cesa2013online}.
	
	In this paper, we aim to study whether the Thompson sampling, an algorithm that has been developed in 1933 \cite{thompson1933likelihood} and successfully applied to a broad range of online optimization problems \cite{chapelle2011empirical,schwartz2017customer} to date can achieve a similar performance in the batched setting. In the Thompson sampling algorithm, the agent chooses an action randomly according to its likelihood of being optimal,	and after receiving feedback, i.e. observing rewards, updates its beliefs about the optimal action. The performance of Thompson sampling has been thoroughly analyzed  in the literature \cite{kaufmann2012thompson,korda2013thompson,abeille2017linear,agrawal2017near, russo2014learning,russo2016information} and is known to achieve the optimal problem-dependent and minimax regret in the classical case. Our goal is to understand whether it can be combined with an adaptive batching strategy and maintain its regret performance when allowed to update its beliefs only at the end of a small number of batches.  Note that the earlier algorithms developed specifically for the batched setting \cite{perchet2016batched,gao2019batched} heavily prioritize exploration in the initial batches to eliminate the possibility that a suboptimal arm is played in the final exponentially larger batches, while Thompson sampling inherently balances between exploration and exploitation by randomly sampling actions according to their probability of being optimal. 
	
Our main contribution in this paper is to show that Thompson sampling, combined with a novel adaptive batching scheme, achieves the optimal problem-dependent performance $O(\log(T))$ and at the same time a minimax regret of order $O(\sqrt{T\log(T)})$ by using $O(\log(T))$ batches independent of the problem instance. This performance is achieved simultaneously by a single strategy without the need to tune it for the target criteria, i.e. problem-dependent or minimax regret.  Our strategy is  an anytime strategy; it operates without the knowledge of time horizon $T$, unlike most of the previously mentioned batched algorithms where $T$ is used both in action selection and the optimization of the batch structure.  We note that the policies designed for minimizing the problem-dependent regret can indeed be turned into anytime algorithms while retaining their $O(\log(T))$ regret and batch complexities with the help of the so called doubling trick in \cite{besson2018doubling}, but the same extension does not hold for minimax policies. This is because even with the best known doubling schemes (exponential or geometric), the minimax policies  either suffer a regret significantly larger  than $\sqrt{T}$ or have their batch complexity increase to $\Omega(\log(T))$ (exponential doubling leads the first conclusion and geometric doubling leads to the second). This implies that our anytime Thompson sampling strategy matches the batch size of $O(\log(T))$ needed for these anytime extensions. We also develop a problem dependent bound on the expected number of batches used by our strategy which is $O(\log\log(T))$. This shows that while our strategy uses $O(\log(T))$ batches in the worst case, similar to previous algorithms, in a given instance of the problem with fixed reward distributions we only need  $O(\log\log(T))$ batches on average. To the best of our knowledge, previous algorithms lack such a refined guarantee on the batch complexity.

		\section{Problem Formulation}\label{prelim}
		\subsection{Notations}
		We denote the natural logarithm as $\log(\cdot)$ while a logarithmic function of base $a>1$ is $\log_a(\cdot)$. For the non-negative sequences of $\{a_n\}$ and $\{b_n\}$, $a_n=O(b_n)$ if and only if $\limsup_{n\rightarrow\infty}\frac{a_n}{b_n}<\infty$ and $a_n=\Omega(b_n)$ if and only if $b_n=O(a_n)$. We also denote by $Q(\cdot)$ the probability of a standard normal random variable $X$ being bigger than a certain threshold $x$, i.e. $Q(x)=\Prob(X\geq x)$ for any $x\in\mathbb{R}$. Finally $\1(\cdot)$ is defined as the indicator function.
		\subsection{Batched Multi-Armed Bandit}\label{setting}
		We consider the batched multi-armed bandit setting. Here there are $K$ arms, where each consecutive pull of the $i^{th}$ arm produces bounded i.i.d. rewards $\{Y_{i,t}\}_{t=1}^\infty$ such that
		\begin{align*}
			Y_{i,1}&\in [0,1],\\
			\Ex[Y_{i,1}]&=\mu_i\in\mathbb{R}.\\
		\end{align*}
		These mean rewards $\{\mu_i\}_{i=1}^K$ are assumed to be deterministic parameters unknown to the agent, whose goal is to accumulate as much reward as possible by repeatedly pulling these arms. Therefore, at each time instance $t$, the agent plays an arm $A_t\in\{1,2,...,K\}$ and receives the reward $Y_{A_t,t}$. Since she can only act causally and does not know $\{\mu_i\}_{i=1}^K$, she can only use the past observations, $\Hp_t=\{A_1,Y_{A_t,1},...,A_t,Y_{A_t,t}\}$ where $\Hp_{0}=\emptyset$, to select the next action $A_{t+1}$.
		
		In this paper, we study the batched version of this multi-armed bandit problem, where the agent has to play these arms in batches and can only incorporate the feedback from the system, i.e. her rewards, into her algorithm at the end of a batch. In other words, there are batch end points $0=T_0<T_1<...$, and the actions the agent plays in the $j^{th}$ batch  $[T_{j-1}+1,T_j]$ as well as the size of the batch itself, i.e. $T_j-T_{j-1}$, can depend only on the information present in $\Hp_{T_{j-1}}$ for any $j\in\mathbb{Z}^+$ and some external randomness that is independent of the system. Note that in this setting, the agent is allowed to adaptively choose the batch sizes depending on the past observations.

		Finally, we let $\mu_1>\mu_i$ for any $i\geq 2$. Given that the agent aims to maximize her cumulative reward, she would only play the first arm if she knew the hidden system parameters $\{\mu_i\}_{i=1}^K$. This observation naturally leads to the cumulative regret term, $R(T)$:
		\begin{equation}
			\Ex[R(T)]=\sum_{i=2}^{K}\Delta_i\Ex[N_i(T)]\label{r9}
		\end{equation}
		where $\mu_1-\mu_i=\Delta_i$ and
		\begin{equation*}
			N_{i}(T)=\sum_{t=1}^{T}\mathbbm{1}(A_t=i)
		\end{equation*}
		for $i\in\{1,2,...,K\}$.

		\section{Batched Thompson Sampling}\label{thompsonsec}
		
		In this section, we describe our Batched Thompson sampling strategy for the batched multi-armed bandit setting described in the previous section. This policy uses Gaussian priors in the spirit of \cite{agrawal2017near} and each arm is sampled randomly according to its likelihood of being optimal under this prior and the observations from previous batches. We combine this strategy with an batching mechanism which relies on the notion of \textit{cycles}. A cycle is defined as follows. The first cycle starts in the beginning of the experiment and ends when the agent selects an action different from the previous actions, i.e. it corresponds to the shortest time interval starting from the beginning of the experiment where two different actions are selected. Then the $j^{th}$ cycle for $j>1$ is defined recursively as the time interval from the end of the $j-1^{th}$ cycle to the first time step where the agent selects an action different from the first action in the cycle. In other words, in each cycle the agent plays exactly two different actions. Consider the following example. Assume that the first seven actions played by the agent are as follows: $A_1=1$, $A_2=1$, $A_3=2$, $A_4=1$, $A_5=3$, $A_6=2$, $A_7=2$. Then the first cycle is $[1,3]$ because only at the third time step the agent selected an arm different from the earlier actions. Similarly, the second cycle is $[4,5]$ where the agent played the first and the third actions. The third cycle that started on $t=6$ has not ended yet because only a single action has been  played so far. We use the concept of a cycle to adaptively decide on the batch size. At the beginning of the $j^{th}$ batch, the Thompson sampling agent checks the number of cycles in which each action $i$ has been played since the beginning of the experiment, denoted $M_i(T_{j-1})$, and sets upper limits $U_{i,j}= \max\{1,\lceil\alpha\times M_i(T_{j-1}) \rceil \}$ for the cycle count of each action. Here $\alpha>1$ is a batch growth factor to be chosen. Throughout the $j^{th}$ batch, the agent employs Thompson sampling, and at the end of each cycle checks whether or not the number of cycles in which a certain arm has been selected since the beginning of the experiment has reached its upper limit $U_{i,j}$ set for the current batch.  The batch ends if there is one such action hitting its upper limit. After the $j^{th}$ batch, the agent observes the rewards of its actions and repeats the same process in the next batch. 
		
		We introduce the following notation to denote the beginning and end of the $k^{th}$ cycle, $C_{b,k}$ and $C_{e,k}$  respectively:
		\begin{equation*}
			C_{b,k} =
			\left\{
			\begin{array}{ll}
				1  & \mbox{if } k=1 \\
				C_{e,k-1}+1 & \mbox{if } k>1
			\end{array}
			\right.
		\end{equation*}
		and
		\begin{equation*}
			C_{e,k}=\min\{t\in \mathbb{Z}^+|A_t\neq A_{t-1}\text{ and }t>C_{b,k}\}
		\end{equation*}
		for any positive integer $k$. As can be seen from these definitions, the interval $[C_{b,k},C_{e,k}]$ describes the $k^{th}$ cycle. We also define $M_i(T)$ and $S_i(T)$ as follows
		\begin{equation*}
			M_i(T)=\sum_{t=1}^T \mathbbm{1}(A_t=i, t=C_{b,k} \text{ or } C_{e,k}\text{ for some }k)
		\end{equation*}
		and 
		\begin{equation*}
			S_i(T)=\sum_{t=1}^T \mathbbm{1}(A_t=i, t=C_{b,k} \text{ or } C_{e,k}\text{ for some }k)Y_{i,t}.
		\end{equation*}
		Here $M_i(T)$ denotes the number of cycles in which the $i^{th}$ action has been selected, while $S_i(T)$ is the sum of rewards the agent received from playing the $i^{th}$ action at either the beginning or the end of a cycle over the duration $T$ of the experiment. 
		Note that whether the condition $\{t=C_{b,k} \text{ or } C_{e,k}\text{ for some }k\}$ is satisfied or not can be verified by checking the actions taken until the time step $t$, i.e. $\{A_j\}_{j=1}^t$. We also define $b(t)=\max\{j\in\mathbb{Z}_{\geq 0}|t-1\geq T_j\}$ as the index of the last batch, and 
		$B(t)=\min\{j\in\mathbb{Z}^+|t\leq T_j\}$ as the batch index of the $t^{th}$ time step.
		We provide a pseudo-code for our Batched Thompson Sampling policy in Algorithm \ref{algo}.

		\begin{algorithm}[H]\label{algo}
			\SetAlgoLined
			\textbf{Input:} Batch growth factor $\alpha>1$, Gaussian variance $\sigma^2$\\
			\textbf{Initialization:} $t=1$, $M_i(0)=0$, $S_i(0)=0$, $U_{i,1}=1$, $j=1$, $T_0=0$.\\
			\While{Experiment Run}{
				\textbf{Sample for Each Arm:} $\theta_i(t)\sim\mathcal{N}(\frac{S_i(T_{j-1})}{1+M_i(T_{j-1})},\frac{\sigma^2}{1+M_i(T_{j-1})})$\\
				\textbf{Play an Arm:} $A_t=\argmax_{i} \theta_i(t)$\\
				\textbf{Update the Pull Count:} $M_{i}(t)\gets M_i(t-1)+\mathbbm{1}(A_t=i, t=C_{b,k} \text{ or } C_{e,k}\text{ for some }k)$
				
				\If{$t=C_{e,k} \text{ for some k } \textbf{\&}\text{ } M_i(t)=U_{i,j} \text{ for some i}$}{
					\textbf{End the Current Batch:} $T_j=t$\\
					\textbf{Receive the Rewards:} $\{Y_{A_t,t}\}_{t=T_{j-1}+1}^{T_j}$\\
					\textbf{Update the Cumulative Rewards for Each Arm:}\\
					$S_i(T_j)=S_i(T_{j-1})+\sum_{l=T_{j-1}+1}^{T_j} \mathbbm{1}(A_l=i, l=C_{b,k} \text{ or } C_{e,k}\text{ for some }k)Y_{i,l}$\\
					\textbf{Update the Upper Limites for the Cycle Counts: }$U_{i,j+1}= \max\{1,\lceil\alpha\times M_i(T_j)\rceil \}$\\	
					\textbf{Update the Batch Index:} $j\gets j+1$\\
				}
				\textbf{Update the Time Index:} $t\gets t+1$\\
			}
			\caption{Batched Thompson Sampling}
		\end{algorithm}
		Note that in  Algorithm \ref{algo}, the posterior distribution from which each arm is selected depends only on $\{M_i(T_{b(t)}),S_i(T_{b(t)})\}_{i=1}^K$, that is we only use the rewards from the first and last action selected in each cycle and ignore the rest of the observed rewards. This is to simplify the analysis in the following sections. However, we can also apply the algorithm by incorporating all the observed rewards, which in general can yield better performance while still maintaining the same batch structure. In that case, $\theta_i(t)$ for any $t$ in the $j^{th}$ batch is drawn instead as 
		\begin{equation*}
			\theta_i(t)\sim\mathcal{N}\Big(\frac{\sum_{t=1}^{T_{j-1}}\1(A_t=i)Y_{i,t}}{1+N_i(T_{j-1})},\frac{\sigma^2}{1+N_i(T_{j-1})}\Big).
		\end{equation*}

	\section{Main Results}
	In this section, we state the main results of our paper. We start with the regret performance of Batched Thompson sampling.

	\begin{theorem}\label{thm1}
		Consider the batched multi-armed bandit setup described in Section \ref{prelim}. If $T\geq 2$ and the batch growth factor $\alpha$ satisfies $\frac{5\sigma^2}{4}\geq\alpha$, then Batched Thompson sampling obeys the following inequalities 
		\begin{equation}
			\Ex[N_i(T)]\leq C_1\alpha\sigma^2\frac{\log(T)}{\Delta_i^2}\label{r3}
		\end{equation}
		for any $i\geq 2$, which lead to 
		\begin{equation}
			\Ex[R(T)]\leq C_1\alpha\sigma^2\sum_{i=2}^K\frac{\log(T)}{\Delta_i}\label{r4},
		\end{equation}
		and
		\begin{equation}
				\Ex[R(T)]\leq C_2\sigma\sqrt{\alpha KT\log(T)}\label{r6}
		\end{equation}
		where $C_1, C_2\geq 1$ are absolute constants independent of the system parameters.
	\end{theorem}
	We provide the proof of this theorem for the special case of $K=2$, $\alpha=2$, $\sigma^2=1$ in Section \ref{pthm1}, and defer the proof of the general version to the appendix. 
	
	Theorem \ref{thm1} states that Batched Thompson sampling achieves $O(\log(T)\sum_{i=2}^K \Delta_i^{-1})$ problem-dependent regret and $O(\sqrt{KT\log(T)})$ minimax regret, which match the asymptotic lower bound of $\Omega(\log(T))$ \cite{lai1985asymptotically}  and the minimax lower bound of $\Omega(\sqrt{KT})$ \cite{audibert2009minimax} up to a $\sqrt{\log(T)}$ term respectively.
	
	We next compare our bounds with the results on classical Thompson sampling \cite{agrawal2017near}. As described in the previous section,  we use Gaussian priors for Thompson sampling following the work of \citet{agrawal2017near}. This is one of the two priors considered in that work for Thompson sampling: beta priors and Gaussian priors. For Thompson sampling with Beta priors, \citet{agrawal2017near} provides two different bounds on the expected regret:
	\begin{equation}
		\Ex[R(T)]\leq(1+\epsilon)\sum_{i=2}^{K}\frac{\log(T)}{d(\mu_i,\mu_1)}\Delta_i+O(\frac{K}{\epsilon^2}),
		\label{r7}
	\end{equation}
	for any $\epsilon\in(0,1)$, and
	\begin{equation}
		\Ex[R(T)]\leq O(\sqrt{KT\log(T)}),\label{r8}
	\end{equation}  
	where $d(\mu_i,\mu_1)=\mu_i\log(\frac{\mu_i}{\mu_1})+(1-\mu_i)\log(\frac{1-\mu_i}{1-\mu_1})$. In addition, they show that  with Gaussian priors the expected regret of the classical Thompson sampling is bounded by $O(\sqrt{KT\log(K)})$ if $T\geq K$. Considering that $d(\mu_i,\mu_1)\geq 2\Delta_i^2$ by Pinsker's inequality, \eqref{r7} provides a tighter performance guarantee than \eqref{r4} in terms of the dependence on the reward distributions, but we note that the minimax performance of Batched Thompson sampling, \eqref{r6}, matches the performance of classical Thompson sampling when the agent receives instantaneous feedback about rewards and can update its policy after each action.  These results show that  Batched Thompson sampling, apart from the dependence on the reward distributions in the problem dependent bound, matches the regret performance in the classical case.
	
	We note that the regret bounds in the theorem depend on the batch growth factor $\alpha$. The regret increases linearly as $\alpha$ grows bigger; this is not surprising since bigger batch sizes mean fewer updates for Batched Thompson sampling. %Similarly, an increase in $\sigma^2$ means that arms with smaller average returns will have bigger probability of being chosen; thus regret will also increase with $\sigma^2$. 
	
	We now present the batch complexity results for our algorithm. 

	\begin{theorem}\label{thm2}
			Consider the batched multi-armed bandit setup described in Section \ref{prelim}. If $T\geq 2$ and $\frac{5\sigma^2}{4}\geq\alpha$, then the batch complexity $B(T)$ of Batched Thompson sampling satisfies the following: 
		\begin{equation}
			B(T)\leq 1+K+K\log_\alpha(\frac{T}{K})\label{r1},
		\end{equation}
	\begin{equation}
		\Ex[B(T)]\leq1+K+\log_\alpha(1+C\alpha\sigma^2\sum_{i=2}^K\frac{\log(T)}{\Delta_i^2}))+\sum_{i=2}^K\log_\alpha(1+C\alpha\sigma^2\frac{\log(T)}{\Delta_i^2})\label{r2},\text{ and}
	\end{equation}
	\begin{equation}
		\Ex[B(T)]\leq1+2C\alpha\sigma^2\sum_{i=2}^K\frac{\log(T)}{\Delta_i^2}\label{r5},
	\end{equation}
where $C$ is an absolute constant independent of the system parameters.
	\end{theorem}
	The proof of this theorem is provided in Section \ref{pthm2}.
	
Theorem \ref{thm2} states three different batch complexity guarantees; the first is a deterministic guarantee on the number of batches while the last two bound the number of batches in expectation. If we consider \eqref{r1}, we see that Batched Thompson sampling uses at most $O(K\log(T))$ many batches regardless of the reward distributions. This result and Theorem \ref{thm1} indicate that Batched Thompson sampling matches the regret performance and the batch complexities of optimal problem-dependent batched algorithms developed in \cite{auer2002finite,gao2019batched,esfandiari2019regret}, which also achieve $O(K\log(T))$ problem-dependent regret with $O(\log(T))$ batch complexity. However, compared with the other optimal algorithms, we show that we can further reduce the batch complexity down to $O(K\log\log(T))$ in \eqref{r2}. This is because our batching strategy uses the information it gathers about the system to adaptively decide on the sizes of the batches while most prior batched algorithms use a static batch structure. We note that Algorithm 1 of \citet{esfandiari2019regret} does use an adaptive batching strategy however that strategy appears to be geared towards  obtaining a tighter regret bound rather than reducing the batch complexity.
	
	We note that the bounds on the number of batches in \eqref{r1} and \eqref{r2} diverge to infinity as $\alpha\downarrow 1$. The bound in \eqref{r5} on the other hand decreases with $\alpha$ and can be relevant when $\alpha$ is chosen very small. We note that  if $\alpha<1+\frac{1}{T}$ for a fixed $T$, then Batched Thompson sampling will only allow one cycle per batch throughout the experiment of duration $T$ and the notion of a batch will coincide with the notion of a cycle. In this extremal case with one cycle per batch, \eqref{r5} shows that Batched Thompson sampling will have $O(\log(T))$ batch complexity, while still satisfying the regret bounds stated in Theorem \ref{thm1}. The $O(K\log\log(T))$ bound on the expected batch complexity in \eqref{r2} is enabled by the fact that for larger $\alpha$ we allow for exponentially more cycles in each batch. As the algorithm proceeds and becomes more confident about the system choosing a suboptimal arm becomes less likely and as a result the cycle durations become inherently larger. At the same time, the algorithm allows for exponentially more cycles in each batch. This double batching strategy in our algorithm, via cycles and batches is key to obtain $O(K\log\log(T))$ guarantee in \eqref{r2} with an anytime strategy. Note that the best previously  available guarantee on batch complexity for an anytime batching strategy is $O(\log(T))$.
	
	\section{Experiments}\label{exp}
	
	In this section, we provide some experimental results on the performance of Batched Thompson sampling where we do not skip samples, i.e. the variant mentioned at the end of Section \ref{thompsonsec}, for different values of $\alpha$, $\{1.00001,1.25,1.5,2\}$, and how they perform against normal Thompson sampling under different reward distributions and action counts when time horizon $T=10^5$ and the sampling variance $\sigma^2=1$. We mainly consider four setups: Bernoulli rewards when $K=2$, Figure \ref{fig} (a); Bernoulli rewards when $K=5$, Figure \ref{fig} (b); Gaussian rewards when $K=2$, Figure \ref{fig} (c); Gaussian rewards when $K=5$, Figure \ref{fig} (d). Finally each figure is the result of an experiment averaged over $10^3$ repeats and the average number of batches used throughout the experiment is rounded up to the nearest integer and reported in the parenthesis to the right of the algorithm names on the figures. The figures show that our batching strategy matches the performance of classical Thompson sampling by using roughly $100$ batches over a time horizon of $T=10^5$.
	
	As can be seen from these figures, Batched Thompson sampling achieves almost the same empirical performance as the normal Thompson sampling when we set $\alpha$ small enough so that there is only one cycle per batch, i.e. $\alpha=1.00001$. We also observe that this Batched Thompson sampling version $\alpha=1.00001$, can have a batch count as small as 15. However when $\alpha$ is very small, the problem independent guarantees in \eqref{r1} and \eqref{r2} become very loose and the number of batches can vary more with the reward distributions. This can be partly observed in the figures: for $\alpha=1.00001$, there is larger variation in the average batch complexity across different reward distributions though in all cases the average numbe of batches remain very small. Increasing $\alpha$ leads  to a more stable batch complexity behavior, at the cost of a small multiplicative regret factor; we observe that the batch count almost remains constant for $\alpha=2$ across different reward distributions.

	\begin{figure}[h]
		\centering
		\subfigure[$Y_{1,t}\sim$Bern(0.75), $Y_{2,t}\sim$Bern(0.25)]{\includegraphics[width=0.47\textwidth]{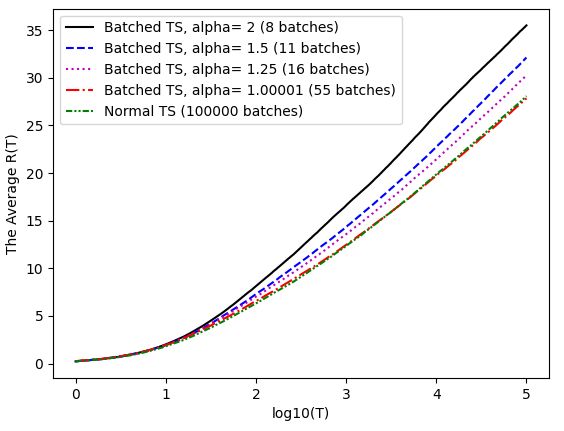}} 
		\subfigure[$Y_{1,t}\sim$Bern(0.75), $Y_{i,t}\sim$Bern(0.25) $2\leq i\leq5$]{\includegraphics[width=0.47\textwidth]{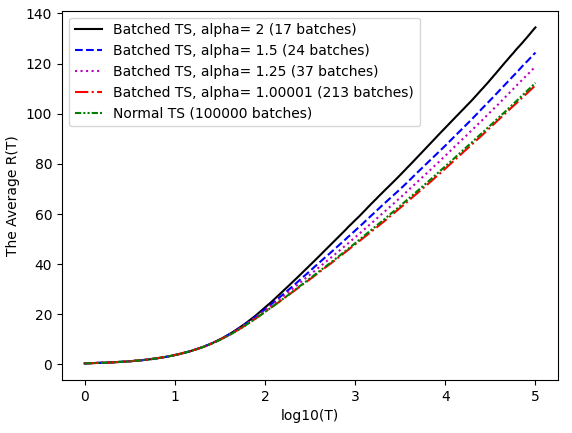}} 
		\subfigure[$Y_{1,t}\sim\mathcal{N}(1,1)$, $Y_{2,t}\sim\mathcal{N}(0,1)$]{\includegraphics[width=0.47\textwidth]{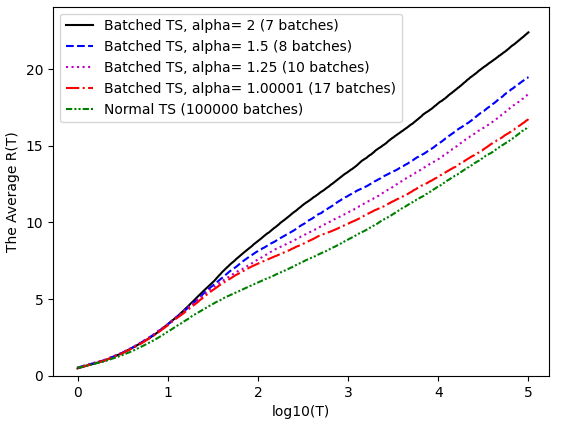}}
		\subfigure[$Y_{1,t}\sim\mathcal{N}(1,1)$, $Y_{i,t}\sim\mathcal{N}(0,1)$ $2\leq i\leq 5$]{\includegraphics[width=0.47\textwidth]{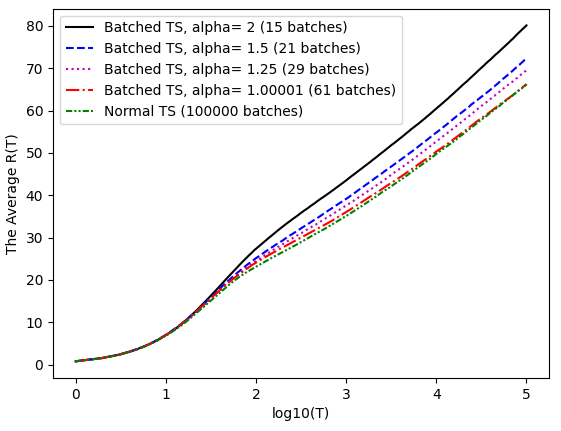}}
		\caption{Empirical Regret Performance of Batched and normal Thompson sampling}
		\label{fig}
	\end{figure}

	\section{Technical Analysis}
	In this section we provide technical proofs for our results. We start with the proof of Theorem \ref{thm1} for a special case and at the end prove Theorem \ref{thm2}.
	\subsection{Proof of Theorem \ref{thm1} when $K=2$, $\alpha=2$, and $\sigma^2=1$}\label{pthm1}

		We first introduce $N_{2,j}(t)$ as the number of times the second arm is pulled if Batched Thompson sampling is employed for $t$ many round with the past knowledge of $\Hp_{T_{j-1}}$. In this case, $N_{2,j}(T_{j}-T_{j-1})=N_2(T_j)-N_2(T_{j-1})$. We know that in the first $T$ rounds, there can be no more than $T$ many batches, and each batch can not last longer than $T$ rounds. As a result, we have the following bound on $N_2(T)$:
		\begin{equation*}
			N_2(T)\leq\sum_{j=1}^{T}\min(N_{2,j}(T_j-T_{j-1}),N_{2,j}(T)).
		\end{equation*}
		Now we first analyze the expected number of times the second is pulled in the first cycle of the $j^{th}$ batch. It is easy to see the time the first arm is selected in this cycle is an upper bound on the number of times the second arm is picked. This observation follows from the fact that if the first action is selected in the first round, then the second is selected only once in the current cycle, if not the time the first arm is selected becomes one more than the number of times the second arm is picked. As a result, conditioned on the past $\Hp_{T_{j-1}}$, the expected number of times the second arm is picked in a single cycle is upper bounded by $\frac{1}{\Prob(A_{T_{j-1}+1}=1|H_{T_{j-1}})}$. Also note that since there are only two actions, each cycle will contain both of them, and in the first $j$ batches, there will be $2^{j-1}$ many cycles by the construction of Algorithm \ref{algo}. This observation means that there are no more than $2^{j-1}$ many identically distributed such cycles in the $j^{th}$ batch, and we have the following bound for any $j$:
		\begin{equation}
			\Ex[N_{2,j}(T_j-T_{j-1})]\leq2^{j-1}\Ex[\frac{1}{\Prob(A_{T_{j-1}+1}=1|H_{T_{j-1}})}]\leq C2^{j-1}\label{ee4}
		\end{equation}
		where $C\geq 2$ is a constant independent of $j$. The last inequality in \eqref{ee4} follows from the following Lemma \ref{lem2} and the fact that in the first batch $\Prob(A_{T_{j-1}+1}=1|H_{T_{j-1}})=\Prob(A_1=1)=\frac{1}{2}$. 
		
		\begin{lemma}\label{lem2}
			If $K=2$, $\alpha=2$, and $\sigma^2=1$, then any $j\geq 2$:
			\begin{align}
				&\Ex\Big[\frac{1}{\Prob(A_{T_{j-1}+1}=1|H_{T_{j-1}})}\Big]\leq C\label{rr1},\\
				&\Prob(A_{T_{j-1}+1}=2)\leq \exp(-\frac{2^{j-4}}{3}\Delta_2^2)\label{rr2},
			\end{align}
		where $C\geq 2$ is a constant independent of $j$. 
		\end{lemma}
		 In addition, we also have:
		\begin{equation}
			\Ex[N_{2,j}(T)]=T\Prob(A_{T_{j-1}+1}=2)\leq T\exp(-\frac{2^{j-4}}{3}\Delta_2^2)\label{ee5}
		\end{equation}
	 	for $j\geq 2$ by the same lemma. The overall analysis shows that for any positive integer $k$, we have:
		\begin{align}
			&\Ex[N_2(T)]\leq\Ex[\sum_{j=1}^{T}\min(N_{2,j}(T_j-T_{j-1}),N_{2,j}(T))]\leq\sum_{j=1}^{k}\Ex[N_{2,j}(T_j-T_{j-1})]+\sum_{j=k+1}^{T}\Ex[N_{2,j}(T)]
			\nonumber\\
			&\leq C(2^{k}-1)+T\sum_{j=k+1}^{T}\exp(-\frac{2^{j-4}}{3}\Delta_2^2)\label{ee6}
		\end{align}
		where the last step follows from \eqref{ee4} and \eqref{ee5}. Let $k$ be the smallest positive integer such that $\frac{2^k}{3}\geq 8\frac{\log(T)}{\Delta_2^2}$. Then we have
		\begin{align*}
			T\sum_{j=k+1}^{T}\exp(-\frac{2^{j-4}}{3}\Delta_2^2)\leq T\sum_{i=0}^{\infty}\exp(-2^i\frac{2^{k-3}}{3}\Delta_2^2)\leq T\sum_{i=0}^{\infty}\exp(-2^i\log(T))
			=\sum_{i=0}^{\infty}\frac{1}{T^{2^i-1}}\leq\sum_{i=0}^{\infty}\frac{1}{T^i}\leq 2
		\end{align*}
		where the last inequality follows from the assumption that $T\geq 2$. This analysis bounds the last summation term in \eqref{ee6}. To bound $C(2^{k}-1)$, note that $k$ is the smallest positive integer bigger than $\log_2(24\frac{\log(T)}{\Delta_2^2})$. Since $\Delta_2^2\leq 1$ and $T\geq 2$, we know $24\frac{\log(T)}{\Delta_2^2}\geq 1$. This analysis shows that $k\leq \log_2(24\frac{\log(T)}{\Delta_2^2})+1=\log_2(48\frac{\log(T)}{\Delta_2^2})$. As a result, $C(2^{k}-1)\leq 48C\frac{\log(T)}{\Delta_2^2}$. The overall analysis shows that $\Ex[R(T)]=\Delta_2\Ex[N_2(T)]\leq 48C\frac{\log(T)}{\Delta_2}+2\Delta_2$ by \eqref{ee6}. This finishes the proof of \eqref{r4}. Finally \eqref{r6} is proven by \eqref{r4} if $\Delta_2>\sqrt{\frac{\log(T)}{T}}$. If not, then $\Ex[R(T)]\leq T\Delta_2\leq \sqrt{T\log(T)}$, and this proves \eqref{r6}.

	\subsection{Proof of Theorem \ref{thm2}}\label{pthm2}

		Let us consider the case where the agent has already employed the Batched Thompson sampling, Algorithm \ref{algo}, for $T$ many steps and denote $i_j\in\{1,2,...,B(T)-1\}$ for $j\in\{1,2,...,k_i\}$ as the indices where $M_i(T_{i_j})=U_{i,i_j}$. Since each batch end point $T_j$ has to satisfy the condition $M_l(T_j)=U_{l,j}$ for some $l$, we have 
		\begin{equation}
			B(T)-1\leq \sum_{i=1}^K k_i.\label{d3}
		\end{equation}
		 By the definition of $U_{i,j}$, we know that $M_i(T_{i_1})=1$. In addition, note that there may be batches in between $i_{j-1}^{th}$ and $i_j^{th}$ ones and the agent may have picked the $i^{th}$ arm while the condition $M_i(T_j)=U_{i,j}$ is not satisfied. These observations lead to $\max\{\alpha M_i(T_{i_{j-1}}),M_i(T_{i_{j-1}})+1\}\leq M_i(T_{i_{j}})$ for $j\geq 2$. The overall analysis shows that if $k_i\geq 1$, then $\alpha^{k_i-1}\leq M_i(T_{i_{k_i}})\leq M_i(T_{B(T)-1})$ due to the fact that $T_{i_{k_i}}\leq T_{B(T)-1}$, which leads to $k_i\leq 1+\log_\alpha(1+M_i(T_{B(T)-1}))$ for any $k_i\geq 0$. In addition, we also have the following trivial bound $k_i\leq  M_i(T_{i_{k_i}})\leq M_i(T_{B(T)-1})$. As a result of these inequalities and \eqref{d3}, we have
		\begin{equation}
			B(T)\leq 1+K+\sum_{i=1}^K \log_\alpha(1+M_i(T_{B(T)-1}))\label{d1}
		\end{equation}
		and 
		\begin{equation}
			B(T)\leq 1+\sum_{i=1}^K M_i(T_{B(T)-1}).\label{d4}
		\end{equation}
		
		First of all, since $\log(\cdot)$ is a concave function, Jensen's inequality and \eqref{d1} lead to
		\begin{equation}
			B(T)\leq 1+K+K\log_\alpha(1+\frac{1}{K}\sum_{i=1}^K M_i(T_{B(T)-1})).\label{d2}
		\end{equation}
		Considering that each cycle has to contain at least two action steps, there can be no more than $\frac{T}{2}$ many cycles in the first $B(T)-1$ batches. In addition, each cycle can only be recorded once by two different actions. This analysis leads to $\sum_{i=1}^K M_i(T_{B(T)-1})\leq T$, which proves \eqref{r1} by \eqref{d2}. To prove \eqref{r2}, we first note that $M_1(T_{B(T)-1})\leq\sum_{i=2}^K M_i(T_{B(T)-1})$ because each cycle of the first arm has to be accompanied by another arm. Since $M_i(T_{B(T)-1})\leq N_i(T)$ for any $i$, we have $B(T)\leq 1+K+\log_\alpha(1+\sum_{i=2}^K N_i(T))+\sum_{i=2}^K\log_\alpha(1+N_i(T))$ by \eqref{d1}, which shows that
		\begin{align}
			\Ex[B(T)]&\leq \Ex[1+K+\log_\alpha(1+\sum_{i=2}^K N_i(T))+\sum_{i=2}^K\log_\alpha(1+N_i(T))]\nonumber\\
			&\leq 1+K+\log_\alpha(1+\sum_{i=2}^K \Ex[N_i(T)])+\sum_{i=2}^K\log_\alpha(1+\Ex[N_i(T)])\nonumber.
		\end{align}
		The last inequality follows from Jensen's inequality. This leads to \eqref{r2} by the fact that $\Ex[N_i(T)]\leq C\alpha\sigma^2\frac{\log(T)}{\Delta_i^2}$ from \eqref{r3}. Finally, the previous analysis also implies that $B(T)\leq 1+2\sum_{i=2}^K N_i(T)$	by \eqref{d4}. This inequality and \eqref{r3} prove \eqref{r5}.

	%%%%%%%%%%%%%%%%%%%%%%%%%%%%%%%%%%%%%%%%%%%%%%%%%%%%%%%%%%%%
	\bibliographystyle{plainnat}
	\bibliography{sample}

\begin{thebibliography}{25}
\providecommand{\natexlab}[1]{#1}
\providecommand{\url}[1]{\texttt{#1}}
\expandafter\ifx\csname urlstyle\endcsname\relax
  \providecommand{\doi}[1]{doi: #1}\else
  \providecommand{\doi}{doi: \begingroup \urlstyle{rm}\Url}\fi

\bibitem[Abeille et~al.(2017)Abeille, Lazaric, et~al.]{abeille2017linear}
Marc Abeille, Alessandro Lazaric, et~al.
\newblock Linear thompson sampling revisited.
\newblock \emph{Electronic Journal of Statistics}, 11\penalty0 (2):\penalty0
  5165--5197, 2017.

\bibitem[Agrawal and Goyal(2017)]{agrawal2017near}
Shipra Agrawal and Navin Goyal.
\newblock Near-optimal regret bounds for thompson sampling.
\newblock \emph{Journal of the ACM (JACM)}, 64\penalty0 (5):\penalty0 1--24,
  2017.

\bibitem[Audibert et~al.(2009)Audibert, Bubeck, et~al.]{audibert2009minimax}
Jean-Yves Audibert, S{\'e}bastien Bubeck, et~al.
\newblock Minimax policies for adversarial and stochastic bandits.
\newblock In \emph{COLT}, volume~7, pages 1--122, 2009.

\bibitem[Auer et~al.(2002)Auer, Cesa-Bianchi, and Fischer]{auer2002finite}
Peter Auer, Nicolo Cesa-Bianchi, and Paul Fischer.
\newblock Finite-time analysis of the multiarmed bandit problem.
\newblock \emph{Machine learning}, 47\penalty0 (2-3):\penalty0 235--256, 2002.

\bibitem[Besson and Kaufmann(2018)]{besson2018doubling}
Lilian Besson and Emilie Kaufmann.
\newblock What doubling tricks can and can't do for multi-armed bandits.
\newblock \emph{arXiv preprint arXiv:1803.06971}, 2018.

\bibitem[Cesa-Bianchi et~al.(2013)Cesa-Bianchi, Dekel, and
  Shamir]{cesa2013online}
Nicolo Cesa-Bianchi, Ofer Dekel, and Ohad Shamir.
\newblock Online learning with switching costs and other adaptive adversaries.
\newblock In \emph{Advances in Neural Information Processing Systems}, pages
  1160--1168, 2013.

\bibitem[Chapelle and Li(2011)]{chapelle2011empirical}
Olivier Chapelle and Lihong Li.
\newblock An empirical evaluation of thompson sampling.
\newblock In \emph{Advances in neural information processing systems}, pages
  2249--2257, 2011.

\bibitem[Durrett(2019)]{durrett2019probability}
Rick Durrett.
\newblock \emph{Probability: theory and examples}, volume~49.
\newblock Cambridge university press, 2019.

\bibitem[Esfandiari et~al.(2019)Esfandiari, Karbasi, Mehrabian, and
  Mirrokni]{esfandiari2019regret}
Hossein Esfandiari, Amin Karbasi, Abbas Mehrabian, and Vahab Mirrokni.
\newblock Regret bounds for batched bandits.
\newblock \emph{arXiv preprint arXiv:1910.04959}, 2019.

\bibitem[Gao et~al.(2019)Gao, Han, Ren, and Zhou]{gao2019batched}
Zijun Gao, Yanjun Han, Zhimei Ren, and Zhengqing Zhou.
\newblock Batched multi-armed bandits problem.
\newblock In \emph{Advances in Neural Information Processing Systems}, pages
  503--513, 2019.

\bibitem[Han et~al.(2020)Han, Zhou, Zhou, Blanchet, Glynn, and
  Ye]{han2020sequential}
Yanjun Han, Zhengqing Zhou, Zhengyuan Zhou, Jose Blanchet, Peter~W Glynn, and
  Yinyu Ye.
\newblock Sequential batch learning in finite-action linear contextual bandits.
\newblock \emph{arXiv preprint arXiv:2004.06321}, 2020.

\bibitem[Hoeffding(1994)]{hoeffding1994probability}
Wassily Hoeffding.
\newblock Probability inequalities for sums of bounded random variables.
\newblock In \emph{The Collected Works of Wassily Hoeffding}, pages 409--426.
  Springer, 1994.

\bibitem[Jin et~al.(2020)Jin, Xu, Xiao, and Gu]{jin2020double}
Tianyuan Jin, Pan Xu, Xiaokui Xiao, and Quanquan Gu.
\newblock Double explore-then-commit: Asymptotic optimality and beyond.
\newblock \emph{arXiv preprint arXiv:2002.09174}, 2020.

\bibitem[Kaufmann et~al.(2012)Kaufmann, Korda, and Munos]{kaufmann2012thompson}
Emilie Kaufmann, Nathaniel Korda, and R{\'e}mi Munos.
\newblock Thompson sampling: An asymptotically optimal finite-time analysis.
\newblock In \emph{International conference on algorithmic learning theory},
  pages 199--213. Springer, 2012.

\bibitem[Korda et~al.(2013)Korda, Kaufmann, and Munos]{korda2013thompson}
Nathaniel Korda, Emilie Kaufmann, and Remi Munos.
\newblock Thompson sampling for 1-dimensional exponential family bandits.
\newblock In \emph{Advances in neural information processing systems}, pages
  1448--1456, 2013.

\bibitem[Lai and Robbins(1985)]{lai1985asymptotically}
Tze~Leung Lai and Herbert Robbins.
\newblock Asymptotically efficient adaptive allocation rules.
\newblock \emph{Advances in applied mathematics}, 6\penalty0 (1):\penalty0
  4--22, 1985.

\bibitem[Perchet et~al.(2016)Perchet, Rigollet, Chassang, Snowberg,
  et~al.]{perchet2016batched}
Vianney Perchet, Philippe Rigollet, Sylvain Chassang, Erik Snowberg, et~al.
\newblock Batched bandit problems.
\newblock \emph{The Annals of Statistics}, 44\penalty0 (2):\penalty0 660--681,
  2016.

\bibitem[Ren and Zhou(2020)]{ren2020dynamic}
Zhimei Ren and Zhengyuan Zhou.
\newblock Dynamic batch learning in high-dimensional sparse linear contextual
  bandits.
\newblock \emph{arXiv preprint arXiv:2008.11918}, 2020.

\bibitem[Ren et~al.(2020)Ren, Zhou, and Kalagnanam]{ren2020batched}
Zhimei Ren, Zhengyuan Zhou, and Jayant~R Kalagnanam.
\newblock Batched learning in generalized linear contextual bandits with
  general decision sets.
\newblock \emph{IEEE Control Systems Letters}, 2020.

\bibitem[Ruan et~al.(2020)Ruan, Yang, and Zhou]{ruan2020linear}
Yufei Ruan, Jiaqi Yang, and Yuan Zhou.
\newblock Linear bandits with limited adaptivity and learning distributional
  optimal design.
\newblock \emph{arXiv preprint arXiv:2007.01980}, 2020.

\bibitem[Russo and Van~Roy(2014)]{russo2014learning}
Daniel Russo and Benjamin Van~Roy.
\newblock Learning to optimize via posterior sampling.
\newblock \emph{Mathematics of Operations Research}, 39\penalty0 (4):\penalty0
  1221--1243, 2014.

\bibitem[Russo and Van~Roy(2016)]{russo2016information}
Daniel Russo and Benjamin Van~Roy.
\newblock An information-theoretic analysis of thompson sampling.
\newblock \emph{The Journal of Machine Learning Research}, 17\penalty0
  (1):\penalty0 2442--2471, 2016.

\bibitem[Schwartz et~al.(2017)Schwartz, Bradlow, and
  Fader]{schwartz2017customer}
Eric~M Schwartz, Eric~T Bradlow, and Peter~S Fader.
\newblock Customer acquisition via display advertising using multi-armed bandit
  experiments.
\newblock \emph{Marketing Science}, 36\penalty0 (4):\penalty0 500--522, 2017.

\bibitem[Thompson(1933)]{thompson1933likelihood}
William~R Thompson.
\newblock On the likelihood that one unknown probability exceeds another in
  view of the evidence of two samples.
\newblock \emph{Biometrika}, 25\penalty0 (3/4):\penalty0 285--294, 1933.

\bibitem[Vershynin(2019)]{vershynin2019high}
Roman Vershynin.
\newblock High-dimensional probability, 2019.

\end{thebibliography}

	\newpage
	\appendix
	
	\section{Outline}
	The appendix is organized as follows.
	\begin{enumerate}
		\item Section \ref{tools} states technical tools necessary for our proofs.
		\item Section \ref{PLEM} provides the proof of Lemma \ref{lem2}, which is stated in Section \ref{pthm1}.
		\item In Section \ref{relevant}, we present couple propositions and lemmas in preparation for the proof of Theorem \ref{thm1} in the general case of $K$, $\alpha$, and $\sigma^2$.
		\item Finally, the full proof of Theorem \ref{thm1} is given in Section \ref{realpthm1}.
	\end{enumerate}
	
	\section{Technical Tools}\label{tools}
	\subsection{Bounded Random Variable Moment-generating Function Bound}
	Let $X$ be a bounded random variable such that $a\leq X\leq b$ and $|a|,|b|<\infty$. \citet{hoeffding1994probability} showed that
	\begin{equation}
		\Ex[\exp(\lambda X)]\leq \exp(\lambda^2\frac{(b-a)^2}{8})\label{hoeff}
	\end{equation}
for any real number $\lambda$.

\subsection{Gaussian Tail Bounds}
 Proposition 2.1.2 of \cite{vershynin2019high} shows that
\begin{equation}
	\Big(\frac{1}{\delta}-\frac{1}{\delta^3}\Big)\frac{\exp(-\frac{\delta^2}{2})}{\sqrt{2\pi}}\leq Q(\delta)\leq \frac{1}{\delta}\frac{\exp(-\frac{\delta^2}{2})}{\sqrt{2\pi}},\label{realtail}
\end{equation}
if $\delta>0$. Since exponential functions decay faster than power functions, there exists $\delta_0$ such that if $\delta\geq \delta_0$, then
\begin{equation*}
	\exp(-3\delta^2/4)\leq Q(\delta)
\end{equation*} 
which leads to
\begin{equation}
	Q^{-1}(1/x)\geq\sqrt{\frac{4}{3}\log(x)}\label{tail}
\end{equation}
if $x\geq x_0$ for some $x_0\geq2$, where $Q^{-1}(\cdot)$ is the inverse function of $Q(\cdot)$. Note that the last inequality follows from setting $\delta=\sqrt{\frac{4}{3}\log(x)}$ and the fact that $Q(\cdot)$ is decreasing.
\subsection{Expectation of Non-negative Random Variables}
Let $X$ be a non-negative random variable, i.e. $X\geq 0$, then 
\begin{equation}
	\Ex[X]=\int_{0}^\infty \Prob(X> x)dx\label{noneg}
\end{equation}
by Lemma 2.2.13 of \cite{durrett2019probability}.
	\section{Proof of Lemma \ref{lem2}}\label{PLEM}

		First of all,
		\begin{align}
			\Ex\Big[\frac{1}{\Prob(A_{T_{j-1}+1}=1|H_{T_{j-1}})}\Big]&=\int_{0}^{\infty}\Prob\Big(\frac{1}{\Prob(A_{T_{j-1}+1}=1|H_{T_{j-1}})}> x\Big)dx\nonumber\\
			&\leq 2+\int_{2}^{\infty}\Prob(\Prob(A_{T_{j-1}+1}=1|H_{T_{j-1}})\leq\frac{1}{x})dx\label{eee1}
		\end{align}
		by \eqref{noneg}. Conditioned on $\Hp_{T_{j-1}}$, $\theta_i(T_{j-1}+1)$ is distributed as $\mathcal{N}(\frac{S_i(T_{j-1})}{1+M_i(T_{j-1})},\frac{1}{1+M_i(T_{j-1})})$. However when $K=2$ and $\alpha=2$, we know that $M_i(T_{j-1})=2^{j-2}$. This overall analysis leads to
		\begin{align*}
			\Prob(A_{T_{j-1}+1}=1|H_{T_{j-1}})&=\Prob\Big(\mathcal{N}\Big(\frac{S_1(T_{j-1})}{1+2^{j-2}},\frac{1}{1+2^{j-2}}\Big)\geq\mathcal{N}\Big(\frac{S_2(T_{j-1})}{1+2^{j-2}},\frac{1}{1+2^{j-2}}\Big)|H_{T_{j-1}}\Big)\\
			&=\Prob\Big(\mathcal{N}\Big(\frac{S_1(T_{j-1})-S_2(T_{j-1})}{1+2^{j-2}},\frac{2}{1+2^{j-2}}\Big)\geq 0|H_{T_{j-1}}\Big)\\
			&=1-\Prob\Big(\mathcal{N}\Big(0,\frac{2}{1+2^{j-2}}\Big)\geq \frac{S_1(T_{j-1})-S_2(T_{j-1})}{1+2^{j-2}}|H_{T_{j-1}}\Big)\\
			&=1-Q\Big(\frac{S_1(T_{j-1})-S_2(T_{j-1})}{\sqrt{2+2^{j-1}}}\Big)
		\end{align*} 
		where the last equality follows from the definition of the function $Q(\cdot)$. Combining the last analysis with \eqref{eee1} shows that
		\begin{align}
			\Ex\Big[\frac{1}{\Prob(A_{T_{j-1}+1}=1|H_{T_{j-1}})}\Big]&\leq2+\int_{2}^{\infty}\Prob\Big(1-Q\Big(\frac{S_1(T_{j-1})-S_2(T_{j-1})}{\sqrt{2+2^{j-1}}}\Big)\leq\frac{1}{x}\Big)dx \nonumber\\
			&=2+\int_{2}^{\infty}\Prob\Big(Q\Big(\frac{S_1(T_{j-1})-S_2(T_{j-1})}{\sqrt{2+2^{j-1}}}\Big)\geq1-\frac{1}{x}\Big)dx\nonumber\\
			&= 2+\int_{2}^{\infty}\Prob\Big(\frac{S_2(T_{j-1})-S_1(T_{j-1})}{\sqrt{2+2^{j-1}}}\geq -Q^{-1}\Big(1-\frac{1}{x}\Big)\Big)dx\label{eee2}\\
			&=2+\int_{2}^{\infty}\Prob\Big(\frac{S_2(T_{j-1})-S_1(T_{j-1})}{\sqrt{2+2^{j-1}}}\geq Q^{-1}\Big(\frac{1}{x}\Big)\Big)dx\label{eee3}
		\end{align}
		where \eqref{eee2} follows from the fact that $Q(\cdot)$ is a decreasing function, and \eqref{eee3} is the result of the symmetric nature of the normal distribution.
	
		Let $\lambda$ be any real number. Here we are going to use induction hypothesis. We know that $\Ex[\exp(\lambda(S_2(T_{1})-S_1(T_{1})+\Delta_2))]\leq \exp(\lambda^2/4)$ by \eqref{hoeff} and the fact that $S_1(T_1)$ and $S_2(T_1)$ are independent bounded random variables. Now assume that $\Ex[\exp(\lambda(S_2(T_{j})-S_1(T_{j})+2^{j-1}\Delta_2))]\leq \exp(2^{j-3}\lambda^2)$ for some $j\geq 1$. However we know that $\Hp_{T_{j}}$, $S_1(T_{j+1})-S_1(T_j)$, and $S_2(T_{j+1})-S_2(T_j)$	are mutually independent. This is because regardless of what the agent observes in the first $j$ batches, i.e. $\Hp_{T_{j}}$, she is going to record rewards from both arms only $2^{j-1}$ numbers times in the $j+1^{th}$ batch and it does not matter at which time indices these are recorded since all the future rewards from any arm are i.i.d. as well. This leads to
		\begin{align}
			&\Ex[\exp(\lambda(S_2(T_{j+1})-S_1(T_{j+1})+2^{j}\Delta_2))]\nonumber\\
			&=\Ex[\exp(\lambda(S_2(T_{j})-S_1(T_{j})+2^{j-1}\Delta_2))\nonumber\\
			&\qquad\qquad\qquad\qquad\Ex[\exp(\lambda(S_2(T_{j+1})-S_2(T_j)-S_1(T_j)+S_1(T_{j+1})+2^{j-1}\Delta_2))|\Hp_{T_{j}}]]\label{eee4}
		\end{align}    
		However from the earlier analysis and the fact that the $j+1^{th}$ batch contains $2^{j-1}$ recorded rewards from each arm, we have 
		\begin{align}
			&\Ex[\exp(\lambda(S_2(T_{j+1})-S_2(T_j)-S_1(T_j)+S_1(T_{j+1})+2^{j-1}\Delta_2))|\Hp_{T_{j}}]\nonumber\\
			&=\Ex[\exp(\lambda(S_2(T_{j+1})-S_2(T_j)-2^{j-1}\mu_2))]\Ex[\exp(-\lambda(S_1(T_{j+1})-S_1(T_j)-2^{j-1}\mu_1))]\nonumber\\
			&\leq \exp(2^{j-3}\lambda^2)\label{eee6}
		\end{align}
		where the last inequality follows from \eqref{hoeff} and the fact that the first $l$ rewards and $l+1^{th}$ reward from the same arm are independent since the rewards are i.i.d. Note that each arm will be picked infinitely often since probability selecting any arm in any batch will be almost surely positive due to using Gaussian distribution to select arms. Finally, \eqref{eee4} and \eqref{eee6}, along with the induction step, shows that
		\begin{equation}
			\Ex[\exp(\lambda(S_2(T_{j-1})-S_1(T_{j-1})+2^{j-2}\Delta_2))]\leq \exp(2^{j-4}\lambda^2)\label{eee7}
		\end{equation}
		for any $j\geq 2$. This result leads to the following bound for any $\lambda\geq 0$ and $x\geq2$:
		\begin{align*}
			\Prob\Big(\frac{S_2(T_{j-1})-S_1(T_{j-1})}{\sqrt{2+2^{j-1}}}\geq Q^{-1}\Big(\frac{1}{x}\Big)\Big)&\leq\Prob\Big(\frac{S_2(T_{j-1})-S_1(T_{j-1})+2^{j-2}\Delta_2}{\sqrt{2+2^{j-1}}}\geq Q^{-1}\Big(\frac{1}{x}\Big)\Big)\\
			&\leq \exp(\lambda^2/8-\lambda Q^{-1}(1/x))
		\end{align*}
		where the last inequality follows from the Chernoff bound. Since $Q^{-1}(1/x)\geq0$ when $x\geq 2$, setting $\lambda=4Q^{-1}(1/x)$ shows that
		\begin{equation*}
			\Prob\Big(\frac{S_2(T_{j-1})-S_1(T_{j-1})}{\sqrt{2+2^{j-1}}}\geq Q^{-1}\Big(\frac{1}{x}\Big)\Big)\leq \exp(-2(Q^{-1}(1/x))^2)
		\end{equation*}
		for any $x\geq 2$. Finally, by \eqref{tail}
		\begin{equation}
			\Prob\Big(\frac{S_2(T_{j-1})-S_1(T_{j-1})}{\sqrt{2+2^{j-1}}}\geq Q^{-1}\Big(\frac{1}{x}\Big)\Big)\leq \frac{1}{x^{8/3}}\label{eee5}
		\end{equation}
		if $x\geq x_0$. Putting \eqref{eee5} back into \eqref{eee3} leads to:
		\begin{equation*}
			\Ex\Big[\frac{1}{\Prob(A_{T_{j-1}+1}=1|H_{T_{j-1}})}\Big]\leq x_0+\int_{x_0}^{\infty}\frac{1}{x^{8/3}}dx\leq x_0+1,
		\end{equation*}
		which proves \eqref{rr1} since $x_0$ is independent of any system parameter. 
		
		 We now prove \eqref{rr2}. Similar to the earlier analysis, we can describe the probability of selecting the second arm as the sample from the second arm being bigger than the first arm's:
		\begin{align*}
			&\Prob(A_{T_{j-1}+1}=2)=\Prob\Big(\mathcal{N}\Big(\frac{S_2(T_{j-1})-S_1(T_{j-1})}{1+2^{j-2}},\frac{2}{1+2^{j-2}}\Big)\geq 0\Big)\\
			&=\Prob\Big(\mathcal{N}\Big(\frac{S_2(T_{j-1})-S_1(T_{j-1})}{1+2^{j-2}}+\frac{2^{j-2}}{1+2^{j-2}}\Delta_2,\frac{2}{1+2^{j-2}}\Big)\geq \frac{2^{j-2}}{1+2^{j-2}}\Delta_2\Big)
		\end{align*}
		In view of \eqref{eee7}, we know that $S_2(T_{j-1})-S_1(T_{j-1})+2^{j-2}\Delta_2$ is sub Gaussian with variance proxy $2^{j-3}$. As a result, $\mathcal{N}(\frac{S_2(T_{j-1})-S_1(T_{j-1})}{1+2^{j-2}}+\frac{2^{j-2}}{1+2^{j-2}}\Delta_2,\frac{2}{1+2^{j-2}})$ has the following variance proxy:
		\begin{equation*}
			\frac{2^{j-3}}{(1+2^{j-2})^2}+\frac{2}{1+2^{j-2}}=\frac{2^{j-3}+2(1+2^{j-2})}{(1+2^{j-2})^2}.
		\end{equation*} 
		This observation and Chernoff bound, which states that $\Prob(X\geq x)\leq \exp(-x^2/(2\sigma^2))$ if $x\geq 0$ and $X$ is sub Gaussian with variance proxy $\sigma^2$, lead to:
		\begin{align*}
			\Prob(A_{T_{j-1}+1}=2)&\leq\exp(-\frac{(1+2^{j-2})^2}{2^{j-2}+4(1+2^{j-2})}\frac{(2^{j-2})^2}{(1+2^{j-2})^2}\Delta_2^2)\\
			&=\exp(-\frac{2^{2j-4}}{2^{j-2}+4(1+2^{j-2})}\Delta_2^2)\\
			&\leq\exp(-\frac{2^{2j-4}}{3\times2^{j}}\Delta_2^2)\\
			&=\exp(-\frac{2^{j-4}}{3}\Delta_2^2)
		\end{align*}
		which finishes the proof of \eqref{rr2}.

	\section{Results Related to Theorem \ref{thm1}}\label{relevant}
	\subsection{Martingale Lemma}
	In this part, we present a key martingale lemma.
	\begin{lemma}\label{lem1}
		Let $\F_t=\{Y_{A_1,1},Y_{A_1,1},...,Y_{A_{T_b(t)},T_{b(t)}},A_1,A_2,...,A_t\}$, then  $X_t=\exp(\lambda(S_i(T_{b(t)})-\mu_iM_i(T_{b(t)}))-\frac{\lambda^2}{8}(1+M_i(T_{b(t)})))$ is a non-negative supermartingale adapted to $\F_t$ for any real $\lambda$ and $i\in\{1,2,..,K\}$. Finally for any $t$ we have 
		\begin{equation}
			\Ex[X_t]\leq 1,\label{e12}
		\end{equation}
		and in particular any stopping time $\tau\leq\infty$ for $\{\F_t\}$ satisfies the following inequality
		\begin{equation}
			\Ex[X_{\tau}]\leq 1,\label{e11}
		\end{equation}
		where $\lim_{t\rightarrow\infty}X_t=X_{\infty}$.
	\end{lemma}
	\begin{proof}
		First of all, it is clear that $X_t$'s are integrable, since the rewards $\{Y_{A_t,t}\}$ are bounded. That means the only thing we need to prove is that $\{X_t\}$ is a supermartingale sequence, i.e. the following inequality
		\begin{equation}
			\Ex[X_{t+1}|\F_t]\leq X_t\label{e10}
		\end{equation} 
		almost surely for any positive integer $t$. By the definition $b(t)$ we know that $T_{b(t)}$ and $T_{b(t+1)}$ are functions of $\{A_1,A_2,...,A_t\}$. Note that the batch end points are decided by the actions taken. In addition, on $\{b(t)=b(t+1)\}$ $X_t$ is equal to $X_{t+1}$. These observations lead to 
		\begin{align}
			&\Ex[X_{t+1}|\F_t]=\Ex[\exp(\lambda(S_i(T_{b(t+1)})-\mu_iM_i(T_{b(t+1)}))-\frac{\lambda^2}{8}(1+M_i(T_{b(t+1)})))|\F_t]\nonumber\\
			&=\1(b(t)=b(t+1))X_t\nonumber\\
			&+\Ex[\1(b(t)\neq b(t+1))\exp(\sum_{j=T_{b(t)}+1}^{T_{b(t+1)}}\1(A_j=i, j=C_{b,k} \text{ or } C_{e,k}\text{ for some }k)(\lambda Y_{i,j}-\lambda\mu_i-\frac{\lambda^2}{8}))|\F_t]X_t\label{e1}
		\end{align}
		where the last equality follows from the definitions of $S_i(T_{b(t)})$ and $M_i(T_{b(t)})$. Note that $b(t)\neq b(t+1)$ if and only if $t$ is a batch end point, which leads to
		\begin{equation}
			\1(b(t)\neq b(t+1))=\sum_{l=0}^{t-1}1(T_{b(t+1)}=t,T_{b(t)}=l).\label{e2}
		\end{equation}
		We first prove that $\1(T_{b(t+1)}=t,T_{b(t)}=l)$ and $\{Y_{i,j}\}_{j=l+1}^t$ are independent. To that end, it is enough to show that
		\begin{equation*}
			\Prob(T_{b(t+1)}=t,T_{b(t)}=l,(Y_{i,l+1},...,Y_{i,t})\in \mathcal{S})=\Prob(T_{b(t+1)}=t,T_{b(t)}=l)\Prob((Y_{i,l+1},...,Y_{i,t})\in \mathcal{S})
		\end{equation*} 
		for any Borel set $\mathcal{S}$ of $\mathbb{R}^{t-l}$. Note that $T_{b(t)}=l$ and $T_{b(t+1)}=t$ if and only if $l$ is a batch end point, call this event $E_1$, and the smallest batch end point strictly bigger than $l$ is $t$, call this event $E_2$. This means $\{T_{b(t)}=l, T_{b(t+1)}=t\}=E_1\cap E_2$. Then we know $\Prob((Y_{i,l+1},...,Y_{i,t})\in \mathcal{S}, E_1)=\Prob((Y_{i,l+1},...,Y_{i,t})\in \mathcal{S})\Prob(E_1)$ due to the fact that whether or not $l$ is a batch end point depends on the past actions $\{A_j\}_{j=1}^l$, which are independent of the future rewards from the $i^{th}$ arm $\{Y_{i,j}\}_{j=l+1}^t$. In addition, conditioned on the fact that there is a batch end point at $l$, i.e. the event $E_1$, different values for $\{Y_{i,j}\}_{j=l+1}^t$ won't change the probability of $E_2$ happening. This is because the agent can not use the information present in $\{Y_{i,j}\}_{j=l+1}^t$ unless the current that started at $l+1$  ends. As a result, we have $\Prob(E_2|E_1,(Y_{i,l+1},...,Y_{i,t})\in \mathcal{S})=\Prob(E_2|E_1)$. The overall analysis leads to
		\begin{align}
			&\Prob(T_{b(t+1)}=t, T_{b(t)}=l, (Y_{i,l+1},...,Y_{i,t})\in \mathcal{S})=\Prob(E_1, E_2, (Y_{i,l+1},...,Y_{i,t})\in \mathcal{S})\nonumber\\
			&=\Prob((Y_{i,l+1},...,Y_{i,t})\in \mathcal{S})\Prob(E_1|(Y_{i,l+1},...,Y_{i,t})\in \mathcal{S})\Prob(E_2|E_1,(Y_{i,l+1},...,Y_{i,t})\in \mathcal{S})\nonumber\\
			&=\Prob((Y_{i,l+1},...,Y_{i,t})\in \mathcal{S})\Prob(E_1)\Prob(E_2|E_1)\nonumber\\
			&=\Prob((Y_{i,l+1},...,Y_{i,t})\in \mathcal{S})\Prob(E_1, E_2)\nonumber\\
			&=\Prob((Y_{i,l+1},...,Y_{i,t})\in \mathcal{S})\Prob(T_{b(t+1)}=t,T_{b(t)}=l),\label{e3}
		\end{align}
		which finishes the proof of the fact that $\1(T_{b(t+1)}=t, T_{b(t)}=l)$ and $\{Y_{i,j}\}_{j=l+1}^t$ are independent. Similar to the earlier analysis, conditioned on $\{T_{b(t+1)}=t, T_{b(t)}=l\}$, $\{A_1, Y_{A_1,1}, A_2, Y_{A_2,2},..., A_{T_{b(t)}}, Y_{A_{T_b(t)},T_{b(t)}}\}$ and $\{Y_{i,j}\}_{j=l+1}^t$ are independent, because the future rewards from the $i^{th}$ arm can not affect the past observations, i.e. actions and rewards. However, conditioned on $\{T_{b(t+1)}=t, T_{b(t)}=l\}$, we know that the actions $\{A_j\}_{j=T_{b(t)}+1}^t$ are sampled according to the information present in $\{A_1, Y_{A_1,1}, A_2, Y_{A_2,2},..., A_{T_{b(t)}}, Y_{A_{T_b(t)},T_{b(t)}}\}$. As a result, conditioned on $\{T_{b(t+1)}=t, T_{b(t)}=l\}$, $\F_t$ and $\{Y_{i,j}\}_{j=l+1}^t$ are independent. This overall analysis shows that for any Borel set $\mathcal{S}$ of $\mathbb{R}^{t-l}$ and any element $\mathcal{G}$ of the sigma algebra generated by $\F_t$ we have
		\begin{align}
			&\Prob(T_{b(t+1)}=t,T_{b(t)}=l, (Y_{i,l+1},...,Y_{i,t})\in \mathcal{S}, \mathcal{G})\nonumber\\
			&=\Prob(T_{b(t+1)}=t,T_{b(t)}=l, (Y_{i,l+1},...,Y_{i,t})\in \mathcal{S})\Prob(\mathcal{G}|T_{b(t+1)}=t,T_{b(t)}=l, (Y_{i,l+1},...,Y_{i,t})\in \mathcal{S})\nonumber\\
			&=\Prob((Y_{i,l+1},...,Y_{i,t})\in \mathcal{S})\Prob(T_{b(t+1)}=t,T_{b(t)}=l)\Prob(\mathcal{G}|T_{b(t+1)}=t,T_{b(t)}=l)\label{e4}\\
			&=\Prob((Y_{i,l+1},...,Y_{i,t})\in \mathcal{S})\Prob(T_{b(t+1)}=t,T_{b(t)}=l,\mathcal{G})\label{e5},
		\end{align}
		where \eqref{e4} follows from \eqref{e3} and the fact that conditioned on $\{T_{b(t+1)}=t, T_{b(t)}=l\}$, $\F_t$ and $\{Y_{i,j}\}_{j=l+1}^t$ are independent. Since $\mathcal{S}$ is arbitrary, $\{Y_{i,j}\}_{j=l+1}^t$ and $\1(T_{b(t+1)}=t, T_{b(t)}=l, \mathcal{G})$ are independent.
		
		Now we go back to \eqref{e2}, and note that $\1(T_{b(t+1)}=t,T_{b(t)}=l)$ can be written as a sum of the terms of the following form
		\begin{equation}
			\1(T_{b(t+1)}=t,T_{b(t)}=l)\prod_{n=l+1}^{t}(\1(a_n=-1)+a_n\times \1(A_n=i, n=C_{b,k} \text{ or } C_{e,k}\text{ for some }k)),\label{e6}
		\end{equation} 
		where $a_n\in\{-1, 1\}$. Then we have for any element $\mathcal{\hat{G}}$ of the sigma algebra generated by $\F_t$ 
		\begin{align}
			&\Ex[\1(T_{b(t+1)}=t,T_{b(t)}=l)(\prod_{n=l+1}^{t}(\1(a_n=-1)+a_n\times \1(A_n=i, n=C_{b,k} \text{ or } C_{e,k}\text{ for some }k)))\1(\mathcal{\hat{G}})\nonumber\\
			&\quad\quad\quad\quad\quad\quad\quad\quad\quad\exp(\sum_{j=T_{b(t)}+1}^{T_{b(t+1)}}\1(A_j=i, j=C_{b,k} \text{ or } C_{e,k}\text{ for some }k)(\lambda Y_{i,j}-\lambda\mu_i-\frac{\lambda^2}{8}))]\nonumber\\
			&=\Ex[\1(T_{b(t+1)}=t,T_{b(t)}=l)(\prod_{n=l+1}^{t}(\1(a_n=-1)+a_n\times \1(A_n=i, n=C_{b,k} \text{ or } C_{e,k}\text{ for some }k)))\1(\mathcal{\hat{G}})\nonumber\\
			&\qquad\qquad\qquad\qquad\qquad\qquad\qquad\qquad\qquad\exp(\sum_{j=l+1}^{t}\1(a_j=1)(\lambda Y_{i,j}-\lambda\mu_i-\frac{\lambda^2}{8}))]\label{e7}			
		\end{align}
		Note that $a_n$'s are deterministic variables and by the earlier analysis, i.e. \eqref{e5}, we know that $\1(T_{b(t+1)}=t,T_{b(t)}=l)(\prod_{n=l+1}^{t}(\1(a_n=-1)+a_n\times \1(A_n=i, n=C_{b,k} \text{ or } C_{e,k}\text{ for some }k)))\1(\mathcal{\hat{G}})$ and $\{Y_{i,j}\}_{j=l+1}^t$ are independent. Then by the fact that $\Ex[\exp(\lambda Y_{i,j}-\lambda\mu_i-\frac{\lambda^2}{8})]\leq 1$ due to \eqref{hoeff}, we have 
		\begin{align}
			&\Ex[\1(T_{b(t+1)}=t,T_{b(t)}=l)(\prod_{n=l+1}^{t}(\1(a_n=-1)+a_n\times \1(A_n=i, n=C_{b,k} \text{ or } C_{e,k}\text{ for some }k)))\1(\mathcal{\hat{G}})\nonumber\\
			&\quad\quad\quad\quad\quad\quad\quad\quad\quad\exp(\sum_{j=T_{b(t)}+1}^{T_{b(t+1)}}\1(A_j=i, j=C_{b,k} \text{ or } C_{e,k}\text{ for some }k)(\lambda Y_{i,j}-\lambda\mu_i-\frac{\lambda^2}{8}))]\nonumber\\
			&\leq \Ex[\1(T_{b(t+1)}=t,T_{b(t)}=l)(\prod_{n=l+1}^{t}(\1(a_n=-1)+a_n\times \1(A_n=i, n=C_{b,k} \text{ or } C_{e,k}\text{ for some }k)))\1(\mathcal{\hat{G}})]\nonumber.
		\end{align}
		Since $\mathcal{\hat{G}}$ is arbitrary, the previous inequality shows that
		\begin{align*}
			&\Ex[\1(T_{b(t+1)}=t,T_{b(t)}=l)(\prod_{n=l+1}^{t}(\1(a_n=-1)+a_n\times \1(A_n=i, n=C_{b,k} \text{ or } C_{e,k}\text{ for some }k)))\\
			&\quad\quad\quad\quad\quad\quad\quad\exp(\sum_{j=T_{b(t)}+1}^{T_{b(t+1)}}\1(A_j=i, j=C_{b,k} \text{ or } C_{e,k}\text{ for some }k)(\lambda Y_{i,j}-\lambda\mu_i-\frac{\lambda^2}{8}))|\F_t]\\
			&\leq \1(T_{b(t+1)}=t,T_{b(t)}=l)(\prod_{n=l+1}^{t}(\1(a_n=-1)+a_n\times \1(A_n=i, n=C_{b,k} \text{ or } C_{e,k}\text{ for some }k)))
		\end{align*}
		almost surely, which leads to
		\begin{align}
				&\Ex[\1(T_{b(t+1)}=t,T_{b(t)}=l)\exp(\sum_{j=T_{b(t)}+1}^{T_{b(t+1)}}\1(A_j=i, j=C_{b,k} \text{ or } C_{e,k}\text{ for some }k)(\lambda Y_{i,j}-\lambda\mu_i-\frac{\lambda^2}{8}))|\F_t]\nonumber\\
			&\leq \1(T_{b(t+1)}=t,T_{b(t)}=l)\label{e8}
		\end{align}
		almost surely by the observation in \eqref{e6}. Finally combining \eqref{e2} and \eqref{e8} proves that
		\begin{align}
			&\Ex[\1(b(t)\neq b(t+1))\exp(\sum_{j=T_{b(t)}+1}^{T_{b(t+1)}}\1(A_j=i, j=C_{b,k} \text{ or } C_{e,k}\text{ for some }k)(\lambda Y_{i,j}-\lambda\mu_i-\frac{\lambda^2}{8}))|\F_t]\nonumber\\
			&\leq \1(b(t)\neq b(t+1))\label{e9}
		\end{align}
		almost surely. This inequality and \eqref{e1} lead to \eqref{e10}. We have showed that $X_t$ is a supermartingale sequence. 
		
		Finally, we prove \eqref{e12} and \eqref{e11}. Firstly, note that $T_{b(1)}=T_0=0$, $S_i(0)=0$, and $M_i(0)=0$, which lead to $\Ex[X_t]\leq\Ex[X_1]=\exp(-\lambda^2/8)\leq 1$ for any $t$ by the properties of supermartingales. Coupling this fact with the following theorem finishes the proof:
		\begin{named}[\textbf{Theorem 4.8.4 of \cite{durrett2019probability}}]\label{thm3}
			If $X_t$ is a non-negative supermartingale and $\tau\leq\infty$ is a stopping time, then $\Ex[X_\tau]\leq \Ex[X_1]$ where $\lim_{t\rightarrow\infty}X_t$ exists and $X_\infty=\lim_{t\rightarrow\infty}X_t$.
		\end{named}
		
	\end{proof}
	\subsection{Estimation Error Bound}
	In this section, we provide a proposition stating that if a certain arms is selected in sufficiently many cycles, then sample $\theta_i$ corresponding to that arm has to be close to the true mean with high probability.
	\begin{proposition}\label{prop1}
		Let $T\geq 2$, then for any positive integer $t$ and $i\in\{1,2,...,K\}$, we have
		\begin{equation}
			\Prob\Big(\theta_1(t)\leq\frac{\mu_1+\mu_i}{2}, M_1(T_{b(t)})\geq 32\sigma^2\frac{\log(T)}{\Delta_i^2}\Big)\leq\frac{2}{T}\label{r11}
		\end{equation} 
		and
		\begin{equation}
			\Prob\Big(\theta_i(t)>\frac{\mu_1+\mu_i}{2}, M_i(T_{b(t)})\geq 32\sigma^2\frac{\log(T)}{\Delta_i^2}\Big)\leq\frac{2}{T}.\label{r22}
		\end{equation} 
	\end{proposition}
\begin{proof}
	We first prove \eqref{r11}. Here we have
	\begin{align}
		&\Prob\Big(\theta_1(t)\leq\frac{\mu_1+\mu_i}{2}, M_1(T_{b(t)})\geq 32\sigma^2\frac{\log(T)}{\Delta_i^2}\Big)\nonumber\\
		&\leq \Prob\Big(\theta_1(t)\leq\frac{\mu_1+\mu_i}{2}, \frac{S_1(T_{b(t)})}{1+M_1(T_{b(t)})}\geq\frac{3\mu_1+\mu_i}{4} ,M_1(T_{b(t)})\geq 32\sigma^2\frac{\log(T)}{\Delta_i^2}\Big)\nonumber\\
		&+\Prob\Big(\frac{S_1(T_{b(t)})}{1+M_1(T_{b(t)})}<\frac{3\mu_1+\mu_i}{4} ,M_1(T_{b(t)})\geq 32\sigma^2\frac{\log(T)}{\Delta_i^2}\Big)\label{e15}
	\end{align} 
	We know that conditioned on $\Hp_{t-1}$, $\theta_1(t)$ is distributed as $\mathcal{N}(\frac{S_1(T_{b(t)})}{1+M_1(T_{b(t)})},\frac{\sigma^2}{1+M_1(T_{b(t)})})$. This fact leads to
	\begin{equation*}
		\Prob(\theta_1(t)\leq\frac{\mu_1+\mu_i}{2}|\Hp_{t-1})=Q\Big(\frac{\sqrt{1+M_1(T_{b(t)})}}{\sigma}\Big(\frac{S_1(T_{b(t)})}{1+M_1(T_{b(t)})}-\frac{\mu_1+\mu_i}{2}\Big)\Big).
	\end{equation*}
 	Then we have
 	\begin{align}
 		&\Prob\Big(\theta_1(t)\leq\frac{\mu_1+\mu_i}{2}, \frac{S_1(T_{b(t)})}{1+M_1(T_{b(t)})}\geq\frac{3\mu_1+\mu_i}{4} ,M_1(T_{b(t)})\geq 32\sigma^2\frac{\log(T)}{\Delta_i^2}\Big)\nonumber\\
 		&=\Ex[	\Prob(\theta_1(t)\leq\frac{\mu_1+\mu_i}{2}|\Hp_{t-1})\1\Big(\frac{S_1(T_{b(t)})}{1+M_1(T_{b(t)})}\geq\frac{3\mu_1+\mu_i}{4} ,M_1(T_{b(t)})\geq 32\sigma^2\frac{\log(T)}{\Delta_i^2}\Big)]\nonumber\\
 		&=\Ex[	Q\Big(\frac{\sqrt{1+M_1(T_{b(t)})}}{\sigma}\Big(\frac{S_1(T_{b(t)})}{1+M_1(T_{b(t)})}-\frac{\mu_1+\mu_i}{2}\Big)\Big)\nonumber\\
 		&\qquad\qquad\qquad\qquad\1\Big(\frac{S_1(T_{b(t)})}{1+M_1(T_{b(t)})}\geq\frac{3\mu_1+\mu_i}{4} ,M_1(T_{b(t)})\geq 32\sigma^2\frac{\log(T)}{\Delta_i^2}\Big)]\label{e13}
 	\end{align}
 	where the first equality follows from the fact that $S_1(T_{b(t)})$ and $M_1(T_{b(t)})$ are measurable with respect to $\Hp_{t-1}$. Since 
 	\begin{equation*}
 		\frac{\sqrt{1+M_1(T_{b(t)})}}{\sigma}\Big(\frac{S_1(T_{b(t)})}{1+M_1(T_{b(t)})}-\frac{\mu_1+\mu_i}{2}\Big)\geq \sqrt{2\log(T)}
 	\end{equation*}
    on $\{\frac{S_1(T_{b(t)})}{1+M_1(T_{b(t)})}\geq\frac{3\mu_1+\mu_i}{4} ,M_1(T_{b(t)})\geq 32\sigma^2\frac{\log(T)}{\Delta_i^2}\}$, \eqref{e13} leads to
    \begin{equation*}
    	\Prob\Big(\theta_1(t)\leq\frac{\mu_1+\mu_i}{2}, \frac{S_1(T_{b(t)})}{1+M_1(T_{b(t)})}\geq\frac{3\mu_1+\mu_i}{4} ,M_1(T_{b(t)})\geq 32\sigma^2\frac{\log(T)}{\Delta_i^2}\Big)\leq Q(\sqrt{2\log(T)}).
    \end{equation*} 
	However, we know that $Q(x)\leq \exp(-x^2/2)$ for $x\geq 1$ by \eqref{realtail}, which results in
	\begin{equation}
		\Prob\Big(\theta_1(t)\leq\frac{\mu_1+\mu_i}{2}, \frac{S_1(T_{b(t)})}{1+M_1(T_{b(t)})}\geq\frac{3\mu_1+\mu_i}{4} ,M_1(T_{b(t)})\geq 32\sigma^2\frac{\log(T)}{\Delta_i^2}\Big)\leq\frac{1}{T}\label{e14}
	\end{equation}
	since $T\geq 2$. We now bound the second term on the right-hand side of \eqref{e15}. Note that
	\begin{align}
		&\Prob\Big(\frac{S_1(T_{b(t)})}{1+M_1(T_{b(t)})}<\frac{3\mu_1+\mu_i}{4} ,M_1(T_{b(t)})\geq 32\sigma^2\frac{\log(T)}{\Delta_i^2}\Big)\nonumber\\
		&=\Prob\Big(\frac{S_1(T_{b(t)})-\mu_1 M_1(T_{b(t)})}{1+M_1(T_{b(t)})}<\frac{3\mu_1+\mu_i}{4}-\mu_1\frac{M_1(T_{b(t)})}{1+M_1(T_{b(t)})} ,M_1(T_{b(t)})\geq 32\sigma^2\frac{\log(T)}{\Delta_i^2}\Big)\nonumber\\
		&=\Prob\Big(\frac{S_1(T_{b(t)})-\mu_1 M_1(T_{b(t)})}{1+M_1(T_{b(t)})}<-\frac{\Delta_i}{4}+\mu_1\frac{1}{1+M_1(T_{b(t)})} ,M_1(T_{b(t)})\geq 32\sigma^2\frac{\log(T)}{\Delta_i^2}\Big)\label{e16}.
	\end{align}
	We know that $\mu_1\leq 1$ and 
	\begin{align*}
		\frac{1}{1+M_1(T_{b(t)})}\leq \frac{1}{M_1(T_{b(t)})}\leq \frac{\Delta_i^2}{32\sigma^2\log(T)}\leq \frac{\Delta_i}{8} 
	\end{align*}
if $M_1(T_{b(t)})\geq 32\sigma^2\frac{\log(T)}{\Delta_i^2}$. Note that the last inequality follows from the fact that $0\leq\Delta_i\leq 1$ and $32\sigma^2\log(T)\geq 8$ for $T\geq 2$ and $\sigma^2\geq 1$. This analysis and \eqref{e16} indicate that
\begin{align}
	&\Prob\Big(\frac{S_1(T_{b(t)})}{1+M_1(T_{b(t)})}<\frac{3\mu_1+\mu_i}{4} ,M_1(T_{b(t)})\geq 32\sigma^2\frac{\log(T)}{\Delta_i^2}\Big)\nonumber\\
	&\leq\Prob\Big(\frac{S_1(T_{b(t)})-\mu_1 M_1(T_{b(t)})}{1+M_1(T_{b(t)})}<-\frac{\Delta_i}{8}, M_1(T_{b(t)})\geq 32\sigma^2\frac{\log(T)}{\Delta_i^2}\Big)\label{e17}.
\end{align}
Then by Lemma \ref{lem1} with $\lambda=-\frac{\Delta_i}{2}$
\begin{align}
	&1\geq\Ex[\exp(-\frac{\Delta_i}{2}(S_1(T_{b(t)})-\mu_1M_1(T_{b(t)}))-\frac{\Delta_i^2}{32}(1+M_1(T_{b(t)})))\nonumber\\
	&\qquad\qquad\qquad\1(\frac{S_1(T_{b(t)})-\mu_1 M_1(T_{b(t)})}{1+M_1(T_{b(t)})}<-\frac{\Delta_i}{8}, M_1(T_{b(t)})\geq 32\sigma^2\frac{\log(T)}{\Delta_i^2})]\nonumber\\
	&\geq \Ex[\exp(\frac{\Delta_i^2}{32}(1+M_1(T_{b(t)})))\1(\frac{S_1(T_{b(t)})-\mu_1 M_1(T_{b(t)})}{1+M_1(T_{b(t)})}<-\frac{\Delta_i}{8}, M_1(T_{b(t)})\label{e18}\geq 32\sigma^2\frac{\log(T)}{\Delta_i^2})]\\
	&\geq T\Prob\Big(\frac{S_1(T_{b(t)})-\mu_1 M_1(T_{b(t)})}{1+M_1(T_{b(t)})}<-\frac{\Delta_i}{8}, M_1(T_{b(t)})\geq 32\sigma^2\frac{\log(T)}{\Delta_i^2}\Big)\label{e19}
\end{align}
where \eqref{e18} from the condition set inside the indicator function, i.e. $\frac{S_1(T_{b(t)})-\mu_1 M_1(T_{b(t)})}{1+M_1(T_{b(t)})}<-\frac{\Delta_i}{8}$. Similarly, the condition $M_1(T_{b(t)})\geq 32\sigma^2\frac{\log(T)}{\Delta_i^2}$ and $\sigma^2\geq 1$ lead to \eqref{e19}. Then combining \eqref{e17} and \eqref{e19} results in
\begin{equation*}
	\Prob\Big(\frac{S_1(T_{b(t)})}{1+M_1(T_{b(t)})}<\frac{3\mu_1+\mu_i}{4} ,M_1(T_{b(t)})\geq 32\sigma^2\frac{\log(T)}{\Delta_i^2}\Big)\leq \frac{1}{T}
\end{equation*} 
which, along with \eqref{e15} and \eqref{e14}, proves \eqref{r11}. We now prove \eqref{r22}. However its proof is almost the same as the proof of \eqref{r11}. Similarly we have
\begin{align}
	&\Prob\Big(\theta_i(t)>\frac{\mu_1+\mu_i}{2}, M_i(T_{b(t)})\geq 32\sigma^2\frac{\log(T)}{\Delta_i^2}\Big)\nonumber\\
	&\leq \Prob\Big(\theta_i(t)>\frac{\mu_1+\mu_i}{2}, \frac{S_i(T_{b(t)})}{1+M_i(T_{b(t)})}\leq\frac{\mu_1+3\mu_i}{4} ,M_i(T_{b(t)})\geq 32\sigma^2\frac{\log(T)}{\Delta_i^2}\Big)\nonumber\\
	&+\Prob\Big(\frac{S_i(T_{b(t)})}{1+M_i(T_{b(t)})}>\frac{\mu_1+3\mu_i}{4} ,M_i(T_{b(t)})\geq 32\sigma^2\frac{\log(T)}{\Delta_i^2}\Big)\label{e20}
\end{align} 
Here conditioned on $\Hp_{t-1}$, $\theta_i(t)$ is distributed as $\mathcal{N}(\frac{S_i(T_{b(t)})}{1+M_i(T_{b(t)})},\frac{\sigma^2}{1+M_i(T_{b(t)})})$. Then it is easy to see that
\begin{equation*}
	\Prob(\theta_i(t)>\frac{\mu_1+\mu_i}{2}|\Hp_{t-1})=Q\Big(\frac{\sqrt{1+M_i(T_{b(t)})}}{\sigma}\Big(\frac{\mu_1+\mu_i}{2}-\frac{S_i(T_{b(t)})}{1+M_i(T_{b(t)})}\Big)\Big),
\end{equation*}
which will lead to 
\begin{equation}
	 \Prob\Big(\theta_i(t)>\frac{\mu_1+\mu_i}{2}, \frac{S_i(T_{b(t)})}{1+M_i(T_{b(t)})}\leq\frac{\mu_1+3\mu_i}{4} ,M_i(T_{b(t)})\geq 32\sigma^2\frac{\log(T)}{\Delta_i^2}\Big)\leq \frac{1}{T}\label{e21}
\end{equation}
by an analysis that is almost the same as the one prior to \eqref{e14}. On the other, the last summand in \eqref{e20} satisfies
\begin{align*}
	&\Prob\Big(\frac{S_i(T_{b(t)})}{1+M_i(T_{b(t)})}>\frac{\mu_1+3\mu_i}{4} ,M_i(T_{b(t)})\geq 32\sigma^2\frac{\log(T)}{\Delta_i^2}\Big)\\
	&=\Prob\Big(\frac{S_i(T_{b(t)})-\mu_iM_i(T_{b(t)})}{1+M_i(T_{b(t)})}>\frac{\mu_1+3\mu_i}{4}-\mu_i\frac{M_i(T_{b(t)})}{1+M_i(T_{b(t)})} ,M_i(T_{b(t)})\geq 32\sigma^2\frac{\log(T)}{\Delta_i^2}\Big)\\
	&\leq \Prob\Big(\frac{S_i(T_{b(t)})-\mu_iM_i(T_{b(t)})}{1+M_i(T_{b(t)})}>\frac{\Delta_i}{8} ,M_i(T_{b(t)})\geq 32\sigma^2\frac{\log(T)}{\Delta_i^2}\Big).
\end{align*}
Similar to the analysis in \eqref{e19}, setting $\lambda=\frac{\Delta_i}{2}$ in Lemma \ref{lem1} leads to
\begin{equation}
	\Prob\Big(\frac{S_i(T_{b(t)})}{1+M_i(T_{b(t)})}>\frac{\mu_1+3\mu_i}{4} ,M_i(T_{b(t)})\geq 32\sigma^2\frac{\log(T)}{\Delta_i^2}\Big)\leq\frac{1}{T}.\label{e22}
\end{equation}
In view of \eqref{e20}, combining \eqref{e21} and \eqref{e22} finishes the proof of \eqref{r22}.
\end{proof}

\subsection{Bounds on Functions of $Q$}
This section provides bounds on various functions of $Q$ function. Before we state our results, we introduce for any $i$
\begin{equation*}
	\hat{\tau}_{i,j} =
	\left\{
	\begin{array}{ll}
		\min\{t\in\mathbb{Z}^+| A_t=i, t=C_{b,k} \text{ or } C_{e,k}\text{ for some }k\}  & \mbox{if } j=1 \\
		\min\{t\in\mathbb{Z}^+| A_t=i, t=C_{b,k} \text{ or } C_{e,k}\text{ for some }k, t>\hat{\tau}_{i,j-1}\} & \mbox{if } j>1.
	\end{array}
	\right.
\end{equation*}
Note that if a set is empty, then $\hat{\tau}_{i,j}$ is set to be infinity. Here $\hat{\tau}_{i,j}$ denotes the time index where we choose the $i^{th}$ arm for $j^{th}$ time at the beginning or at the end of a cycle. It is clear that $\hat{\tau}_{i,j}$ is a stopping time for $\{\F_t\}$ specified in Lemma \ref{lem1}, so \eqref{e11} in Lemma \ref{lem1} remains true if we set $\tau$ to be $\hat{\tau}_{i,j}$:
\begin{equation}
	\Ex[\lambda(S_i(T_{b(\hat{\tau}_{i,j})})-\mu_iM_i(T_{b(\hat{\tau}_{i,j})}))-\frac{\lambda^2}{8}(1+M_i(T_{b(\hat{\tau}_{i,j})})))]\leq 1\label{ee1}
\end{equation}
for any $i\in\{1,2,...,K\}$ and real $\lambda$.
\begin{lemma}\label{lem3}
	For any $j>1$ and $i\in\{1,2,...,K\}$, if $x\geq 0$:
	\begin{equation}
		\Prob\Big(\hat{\tau}_{i,j}<\infty,\frac{S_i(T_{b(\hat{\tau}_{i,j})})-\mu_iM_i(T_{b(\hat{\tau}_{i,j})})}{\sqrt{1+M_i(T_{b(\hat{\tau}_{i,j})})}}
		>x\Big)\leq\exp(-2x^2/\alpha),\label{ee2}
	\end{equation}
	and if $x\leq 0$:
	\begin{equation}
		\Prob\Big(\hat{\tau}_{i,j}<\infty,\frac{S_i(T_{b(\hat{\tau}_{i,j})})-\mu_iM_i(T_{b(\hat{\tau}_{i,j})})}{\sqrt{1+M_i(T_{b(\hat{\tau}_{i,j})})}}
		<x\Big)\leq\exp(-2x^2/\alpha).\label{ee3}
	\end{equation}
\end{lemma}
\begin{proof}
	We start with the proof of \eqref{ee2}. Assume $x\geq 0$. By setting $\lambda=\frac{4x}{\sqrt{\alpha j}}$ in \eqref{ee1}, we have
	\begin{align}
		&1\geq\Ex[\exp(\frac{4x}{\sqrt{\alpha j}}(S_i(T_{b(\hat{\tau}_{i,j})})-\mu_iM_i(T_{b(\hat{\tau}_{i,j})}))-\frac{2x^2}{\alpha j}(1+M_i(T_{b(\hat{\tau}_{i,j})})))]\nonumber\\
		&\geq\Ex[\exp(\frac{4x}{\sqrt{\alpha j}}(S_i(T_{b(\hat{\tau}_{i,j})})-\mu_iM_i(T_{b(\hat{\tau}_{i,j})}))-\frac{2x^2}{\alpha j}(1+M_i(T_{b(\hat{\tau}_{i,j})})))\nonumber\\
		&\qquad\qquad\qquad\qquad\qquad\qquad\qquad\1\Big(\hat{\tau}_{i,j}<\infty,\frac{S_i(T_{b(\hat{\tau}_{i,j})})-\mu_1M_i(T_{b(\hat{\tau}_{i,j})})}{\sqrt{1+M_i(T_{b(\hat{\tau}_{i,j})})}}
		>x\Big)]\nonumber\\
		&\geq \Ex[\exp(\frac{4x}{\sqrt{\alpha j}}(S_i(T_{b(\hat{\tau}_{i,j})})-\mu_iM_i(T_{b(\hat{\tau}_{i,j})}))-\frac{2x^2}{\alpha})\1(\hat{\tau}_{i,j}<\infty,\frac{S_i(T_{b(\hat{\tau}_{i,j})})-\mu_1M_i(T_{b(\hat{\tau}_{i,j})})}{\sqrt{1+M_i(T_{b(\hat{\tau}_{i,j})})}}
		>x)]\label{ee7}	
	\end{align}
	Note that the last inequality follows from $M_i(T_{b(\hat{\tau}_{i,j})})\leq j-1$ since $T_{b(\hat{\tau}_{i,j})}$ is the last batch end point, which is strictly smaller than $\hat{\tau}_{i,j}$. Also by the construction of Algorithm \ref{algo}, $\max\{1,\lceil\alpha M_i(T_{b(\hat{\tau}_{i,j})})\rceil\}\geq j$. Given that $j>1$, $M_i(T_{b(\hat{\tau}_{i,j})})\geq 1$, which results in $\alpha(1+M_i(T_{b(\hat{\tau}_{i,j})}))\geq j$ by $\alpha>1$. Then on $\{\hat{\tau}_{i,j}<\infty,\frac{S_i(T_{b(\hat{\tau}_{i,j})})-\mu_iM_i(T_{b(\hat{\tau}_{i,j})})}{\sqrt{1+M_i(T_{b(\hat{\tau}_{i,j})})}}
	>x\}$, we have
	\begin{equation*}
		\frac{4x}{\sqrt{\alpha j}}(S_i(T_{b(\hat{\tau}_{i,j})})-\mu_iM_i(T_{b(\hat{\tau}_{i,j})}))\geq \frac{4x^2}{\alpha}.
	\end{equation*}
	Using this inequality inside \eqref{ee7} proves \eqref{ee2}. Finally, the proof of \eqref{ee3} will follow the same steps with a single exception: here we set $\lambda=-\frac{4x}{\sqrt{\alpha j}}$ in \eqref{ee1}. This finishes the proof.
	
\end{proof}

\begin{proposition}\label{prop2}
Assume $5\sigma^2/4\geq\alpha$. Then for any positive integer $j$ and $i\geq 2$, we have
	\begin{equation}
		\Ex\Bigg[\1(\hat{\tau}_{1,j}<\infty)\frac{1}{Q^2\Bigg(\frac{\sqrt{1+M_1(T_{b(\hat{\tau}_{1,j})})}}{\sigma}\Big(\frac{\mu_1+\mu_i}{2}-\frac{S_1(T_{b(\hat{\tau}_{1,j})})}{1+M_1(T_{b(\hat{\tau}_{1,j})})}\Big)
			\Bigg)}\Bigg]\leq C\label{r111}
	\end{equation}
	and
	\begin{equation}
		\Ex\Bigg[\1(\hat{\tau}_{i,j}<\infty)\frac{1}{Q^2\Bigg(\frac{\sqrt{1+M_i(T_{b(\hat{\tau}_{i,j})})}}{\sigma}\Big(\frac{S_i(T_{b(\hat{\tau}_{i,j})})}{1+M_i(T_{b(\hat{\tau}_{i,j})})}-\frac{\mu_1+\mu_i}{2}\Big)
			\Bigg)}\Bigg]\leq C\label{r222}
	\end{equation}
where $C$ is an absolute constant independent of the system parameters.
\end{proposition}
\begin{proof}
	We start with the proof of \eqref{r111} and the proof of \eqref{r222} will follow similarly. First note that if $j=1$, then $M_1$ and $S_1$ terms will be zeros, and as a result, the denominator inside the expectation is lower bounded as follows:
	\begin{equation*}
		Q(\frac{\mu_1+\mu_i}{2\sigma})\geq Q(0.5)
	\end{equation*}
	since $\mu_i, \mu_1\leq 1$ and $\sigma^2\geq 1$. So for $j=1$, we can choose $C$ to be $1/Q^2(0.5)$. 
	
	For $j>1$ by \eqref{noneg}
	\begin{align}
		&\Ex\Bigg[\1(\hat{\tau}_{1,j}<\infty)\frac{1}{Q^2\Bigg(\frac{\sqrt{1+M_1(T_{b(\hat{\tau}_{1,j})})}}{\sigma}\Big(\frac{\mu_1+\mu_i}{2}-\frac{S_1(T_{b(\hat{\tau}_{1,j})})}{1+M_1(T_{b(\hat{\tau}_{1,j})})}\Big)
			\Bigg)}\Bigg]\nonumber\\
		&=\int_{x=0}^{\infty}\Prob\Bigg(1(\hat{\tau}_{1,j}<\infty)\frac{1}{Q^2\Bigg(\frac{\sqrt{1+M_1(T_{b(\hat{\tau}_{1,j})})}}{\sigma}\Big(\frac{\mu_1+\mu_i}{2}-\frac{S_1(T_{b(\hat{\tau}_{1,j})})}{1+M_1(T_{b(\hat{\tau}_{1,j})})}\Big)
			\Bigg)}>x\Bigg)dx\nonumber\\
		&\leq 4+\int_{x=4}^\infty \Prob\Bigg(\hat{\tau}_{1,j}<\infty,Q\Bigg(\frac{\sqrt{1+M_1(T_{b(\hat{\tau}_{1,j})})}}{\sigma}\Big(\frac{\mu_1+\mu_i}{2}-\frac{S_1(T_{b(\hat{\tau}_{1,j})})}{1+M_1(T_{b(\hat{\tau}_{1,j})})}\Big)
		\Bigg)<\frac{1}{\sqrt{x}}\Bigg)dx\nonumber\\
		&\leq 4+\int_{x=4}^\infty \Prob\Bigg(\hat{\tau}_{1,j}<\infty,\frac{\sqrt{1+M_1(T_{b(\hat{\tau}_{1,j})})}}{\sigma}\Big(\frac{\mu_1+\mu_i}{2}-\frac{S_1(T_{b(\hat{\tau}_{1,j})})}{1+M_1(T_{b(\hat{\tau}_{1,j})})}\Big)
		>Q^{-1}(1/\sqrt{x})\Bigg)dx\label{ee8},
	\end{align}
	where the last inequality follows from the fact that $Q(\cdot)$ is a decreasing function. Now note that
	\begin{align*}
		&\frac{\mu_1+\mu_i}{2}-\frac{S_1(T_{b(\hat{\tau}_{1,j})})}{1+M_1(T_{b(\hat{\tau}_{1,j})})}=\frac{\mu_1+\mu_i}{2}-\mu_1+\mu_1-\frac{\mu_1M_1(T_{b(\hat{\tau}_{1,j})})}{1+M_1(T_{b(\hat{\tau}_{1,j})})}-\frac{S_1(T_{b(\hat{\tau}_{1,j})})-\mu_1M_1(T_{b(\hat{\tau}_{1,j})})}{1+M_1(T_{b(\hat{\tau}_{1,j})})}\\
		&\leq\frac{1}{1+M_1(T_{b(\hat{\tau}_{1,j})})}-\frac{S_1(T_{b(\hat{\tau}_{1,j})})-\mu_1M_1(T_{b(\hat{\tau}_{1,j})})}{1+M_1(T_{b(\hat{\tau}_{1,j})})}.
	\end{align*}
	The last inequality follows from $\mu_1\leq 1$. This analysis shows that if $\sqrt{1+M_1(T_{b(\hat{\tau}_{1,j})})}(\frac{\mu_1+\mu_i}{2}-\frac{S_1(T_{b(\hat{\tau}_{1,j})})}{1+M_1(T_{b(\hat{\tau}_{1,j})})})>\sigma Q^{-1}(1/\sqrt{x})$, then $-\frac{S_1(T_{b(\hat{\tau}_{1,j})})-\mu_1M_1(T_{b(\hat{\tau}_{1,j})})}{\sqrt{1+M_1(T_{b(\hat{\tau}_{1,j})})}}>\sigma Q^{-1}(1/\sqrt{x})-\frac{1}{\sqrt{1+M_1(T_{b(\hat{\tau}_{1,j})})}}$. Since $M_1(T_{b(\hat{\tau}_{1,j})})\geq 0$ and $\sigma^2\geq 1$, the last inequality leads to

	\begin{align}
		&\Ex\Bigg[\1(\hat{\tau}_{1,j}<\infty)\frac{1}{Q^2\Bigg(\frac{\sqrt{1+M_1(T_{b(\hat{\tau}_{1,j})})}}{\sigma}\Big(\frac{\mu_1+\mu_i}{2}-\frac{S_1(T_{b(\hat{\tau}_{1,j})})}{1+M_1(T_{b(\hat{\tau}_{1,j})})}\Big)
			\Bigg)}\Bigg]\nonumber\\
		&\leq 4+\int_{x=4}^\infty \Prob\Big(\hat{\tau}_{1,j}<\infty,-\frac{S_1(T_{b(\hat{\tau}_{1,j})})-\mu_1M_1(T_{b(\hat{\tau}_{1,j})})}{\sqrt{1+M_1(T_{b(\hat{\tau}_{1,j})})}}
		>\sigma(Q^{-1}(1/\sqrt{x})-1)\Big)dx\label{ee9}
	\end{align}
	by \eqref{ee8}. However, \eqref{tail} indicates that there exists $x_0$ such that if $x\geq x_0$, then $Q^{-1}(\frac{1}{\sqrt{x}})-1\geq\sqrt{\frac{16}{25}\log(x)}$, which leads to 
	\begin{equation*}
		\Prob\Big(\hat{\tau}_{1,j}<\infty,-\frac{S_1(T_{b(\hat{\tau}_{1,j})})-\mu_1M_1(T_{b(\hat{\tau}_{1,j})})}{\sqrt{1+M_1(T_{b(\hat{\tau}_{1,j})})}}
		>\sigma(Q^{-1}(1/\sqrt{x})-1)\Big)\leq\exp(-\frac{32\sigma^2}{25\alpha}\log(x))
	\end{equation*}
	by Lemma \ref{lem3} for $x\geq x_0$. Since $\alpha\leq\frac{5\sigma^2}{4}$, the last inequality can be refined to 
	\begin{equation}
		\Prob\Big(\hat{\tau}_{1,j}<\infty,-\frac{S_1(T_{b(\hat{\tau}_{1,j})})-\mu_1M_1(T_{b(\hat{\tau}_{1,j})})}{\sqrt{1+M_1(T_{b(\hat{\tau}_{1,j})})}}
		>\sigma(Q^{-1}(1/\sqrt{x})-1)\Big)\leq\frac{1}{x^{128/125}}\label{ee10}
	\end{equation}
for any $x\geq x_0$. Finally, \eqref{ee9} and \eqref{ee10} shows that
	\begin{align*}
		\Ex\Bigg[\1(\hat{\tau}_{1,j}<\infty)\frac{1}{Q^2\Bigg(\frac{\sqrt{1+M_1(T_{b(\hat{\tau}_{1,j})})}}{\sigma}\Big(\frac{\mu_1+\mu_i}{2}-\frac{S_1(T_{b(\hat{\tau}_{1,j})})}{1+M_1(T_{b(\hat{\tau}_{1,j})})}\Big)
			\Bigg)}\Bigg]\leq x_0+\int_{x=x_0}^\infty\frac{1}{x^{128/125}}dx
	\end{align*}
	This result proves \eqref{r111} since $x_0$ is an absolute constant independent of the system parameters and the final integral is finite. Note that we already upper bounded the $j=1$ term earlier by an absolute constant so we can just take the maximum of the two. 
	
	The proof of \eqref{r222} will similarly follow. First of all, for $j=1$ we can bound the $Q$ terms inside the expectation as follows
	\begin{equation*}
		Q(-\frac{\mu_1+\mu_i}{2\sigma})\geq Q(0)=0.5
	\end{equation*}
	where these follow from the fact that expected means are non-negative. So for $j=1$, we can choose $C$ to be 4. If $j>1$, we have 
	\begin{align}
		&\Ex\Bigg[\1(\hat{\tau}_{i,j}<\infty)\frac{1}{Q^2\Bigg(\frac{\sqrt{1+M_i(T_{b(\hat{\tau}_{i,j})})}}{\sigma}\Big(\frac{S_i(T_{b(\hat{\tau}_{i,j})})}{1+M_i(T_{b(\hat{\tau}_{i,j})})}-\frac{\mu_1+\mu_i}{2}\Big)
			\Bigg)}\Bigg]\nonumber\\
		&\leq 4+\int_{x=4}^\infty \Prob\Bigg(\hat{\tau}_{i,j}<\infty,\frac{\sqrt{1+M_i(T_{b(\hat{\tau}_{i,j})})}}{\sigma}\Big(\frac{S_i(T_{b(\hat{\tau}_{i,j})})}{1+M_i(T_{b(\hat{\tau}_{i,j})})}-\frac{\mu_1+\mu_i}{2}\Big)
		>Q^{-1}(1/\sqrt{x})\Bigg)dx\label{ee11}.
	\end{align}
	As for the terms inside:
	\begin{align*}
		&\frac{S_i(T_{b(\hat{\tau}_{i,j})})}{1+M_i(T_{b(\hat{\tau}_{i,j})})}-\frac{\mu_1+\mu_i}{2}=\frac{S_i(T_{b(\hat{\tau}_{i,j})})-\mu_iM_i(T_{b(\hat{\tau}_{i,j})})}{1+M_i(T_{b(\hat{\tau}_{i,j})})}+\frac{\mu_iM_i(T_{b(\hat{\tau}_{i,j})})}{1+M_i(T_{b(\hat{\tau}_{i,j})})}-\mu_i+\mu_i-\frac{\mu_1+\mu_i}{2}\\
		&\leq \frac{S_i(T_{b(\hat{\tau}_{i,j})})-\mu_iM_i(T_{b(\hat{\tau}_{i,j})})}{1+M_i(T_{b(\hat{\tau}_{i,j})})},
	\end{align*} 
	where the last step follows from $\mu_i\geq 0$. This analysis and \eqref{ee11} lead to
	\begin{align*}
		&\Ex\Bigg[\1(\hat{\tau}_{i,j}<\infty)\frac{1}{Q^2\Bigg(\frac{\sqrt{1+M_i(T_{b(\hat{\tau}_{i,j})})}}{\sigma}\Big(\frac{S_i(T_{b(\hat{\tau}_{i,j})})}{1+M_i(T_{b(\hat{\tau}_{i,j})})}-\frac{\mu_1+\mu_i}{2}\Big)
			\Bigg)}\Bigg]\nonumber\\
		&\leq 4+\int_{x=4}^\infty \Prob\Bigg(\hat{\tau}_{i,j}<\infty,\frac{S_i(T_{b(\hat{\tau}_{i,j})})-\mu_iM_i(T_{b(\hat{\tau}_{i,j})})}{\sqrt{1+M_i(T_{b(\hat{\tau}_{i,j})})}}
		>\sigma Q^{-1}(1/\sqrt{x})\Bigg)dx.
	\end{align*}
	The rest of the proof follows exactly the same way it did after \eqref{ee9} in the proof of \eqref{r111}. 
\end{proof}

\begin{proposition}\label{prop3}
Assume $5\sigma^2/4\geq\alpha$, $T\geq 2$, and $j>1024\alpha\sigma^2\frac{\log(T)}{\Delta_i^2}$, then
	\begin{equation*}
		\Ex\Bigg[\1(\hat{\tau}_{1,j}<\infty)\Bigg(\frac{1}{Q\Bigg(\frac{\sqrt{1+M_1(T_{b(\hat{\tau}_{1,j})})}}{\sigma}\Big(\frac{\mu_1+\mu_i}{2}-\frac{S_1(T_{b(\hat{\tau}_{1,j})})}{1+M_1(T_{b(\hat{\tau}_{1,j})})}\Big)
			\Bigg)}-1\Bigg)^2\Bigg]\leq \frac{C}{T}
	\end{equation*}
	where $C$ is an absolute constant independent of the system parameters.
\end{proposition}

\begin{proof}
	Since $j>1$, we know that $j\leq \lceil\alpha\times M_1(T_{b(\hat{\tau}_{1,j})})\rceil$ by the construction of Algorithm \ref{algo}, and as a result, 
	\begin{equation*}
		1024\alpha\sigma^2\frac{\log(T)}{\Delta_i^2}<j\leq\lceil\alpha\times M_1(T_{b(\hat{\tau}_{1,j})})\rceil\leq2\alpha M_1(T_{b(\hat{\tau}_{1,j})}),
	\end{equation*} 
	which leads to $M_1(T_{b(\hat{\tau}_{1,j})})>512\sigma^2\frac{\log(T)}{\Delta_i^2}$. Note that $M_1(T_{b(\hat{\tau}_{1,j})})\geq 1$ if $j>1$. Then by \eqref{noneg} we have
	\begin{align}
		&\Ex\Bigg[\1(\hat{\tau}_{1,j}<\infty)\Bigg(\frac{1}{Q\Bigg(\frac{\sqrt{1+M_1(T_{b(\hat{\tau}_{1,j})})}}{\sigma}\Big(\frac{\mu_1+\mu_i}{2}-\frac{S_1(T_{b(\hat{\tau}_{1,j})})}{1+M_1(T_{b(\hat{\tau}_{1,j})})}\Big)
			\Bigg)}-1\Bigg)^2\Bigg]\nonumber\\
		&=\int_{x=0}^\infty \Prob\Bigg(\1(\hat{\tau}_{1,j}<\infty)\Bigg(\frac{1}{Q\Bigg(\frac{\sqrt{1+M_1(T_{b(\hat{\tau}_{1,j})})}}{\sigma}\Big(\frac{\mu_1+\mu_i}{2}-\frac{S_1(T_{b(\hat{\tau}_{1,j})})}{1+M_1(T_{b(\hat{\tau}_{1,j})})}\Big)
			\Bigg)}-1\Bigg)^2>x\Bigg)dx\nonumber\\
		&=\int_{x=0}^\infty \Prob\Bigg(\hat{\tau}_{1,j}<\infty,Q\Bigg(\frac{\sqrt{1+M_1(T_{b(\hat{\tau}_{1,j})})}}{\sigma}\Big(\frac{\mu_1+\mu_i}{2}-\frac{S_1(T_{b(\hat{\tau}_{1,j})})}{1+M_1(T_{b(\hat{\tau}_{1,j})})}\Big)
		\Bigg)<\frac{1}{1+\sqrt{x}}\Bigg)dx\nonumber\\
		&=\int_{x=0}^\infty \Prob\Bigg(\hat{\tau}_{1,j}<\infty,\frac{\sqrt{1+M_1(T_{b(\hat{\tau}_{1,j})})}}{\sigma}\Big(\frac{\mu_1+\mu_i}{2}-\frac{S_1(T_{b(\hat{\tau}_{1,j})})}{1+M_1(T_{b(\hat{\tau}_{1,j})})}\Big)>Q^{-1}(\frac{1}{1+\sqrt{x}})\Bigg)dx\label{ee12}.
	\end{align}
	These steps follow from simple algebra and the decreasing nature of $Q(\cdot)$. 
	
	Now the goal here is to divide the integral in \eqref{ee12} into three regions, where the the contribution from each is in the order of $1/T$. As we will show next, the first region $[0,1/T]$ will be upper bounded by $1/T$. In the second region $[1/T,x_0]$ for some big $x_0$, the probability term will be of $O(1/T)$. Finally, the integrand in the third region $[x_0,\infty)$ will decay faster than $\frac{1}{Tx^\delta}$ for some $\delta>1$, and this will result in a contribution of order $O(1/T)$. The fact that $x_0$ is an absolute constant will finish the proof. We will start the analysis with the second region. First of all, the symmetry of the normal distribution leads to
	\begin{equation}
		Q^{-1}(\frac{1}{1+\sqrt{x}})=-Q^{-1}(\frac{\sqrt{x}}{1+\sqrt{x}})\label{ee13}
	\end{equation}
	for $x\geq0$. Given that $Q(y)\leq \exp(-y^2/2)$ for $y\geq 0$ by the Chernoff bound, we also know $Q^{-1}(\exp(-y^2/2))\leq y$ for $y\geq 0$. Clearly $\sqrt{2\log(1+\sqrt{T})}\geq 0$ and letting $y=\sqrt{2\log(1+\sqrt{T})}$ shows that 
	\begin{equation*}
		Q^{-1}(\frac{\sqrt{x}}{1+\sqrt{x}})\leq \sqrt{2\log(1+\sqrt{T})}
	\end{equation*}
	if $x=1/T$. Since $Q^{-1}(\frac{1}{1+\sqrt{x}})$ is an increasing function of $x$, for $x\geq 1/T$, we have
	\begin{equation*}
		Q^{-1}(\frac{1}{1+\sqrt{x}})\geq -\sqrt{2\log(1+\sqrt{T})}.
	\end{equation*}
	by \eqref{ee13}. This analysis proves that
	\begin{equation}
		Q^{-1}(\frac{1}{1+\sqrt{x}})+2\sqrt{2\log(T)}\geq 0\label{ee20}
	\end{equation} 
	 for $x\geq 1/T$ since $T\geq 2$. In addition, by simple algebra
	\begin{align}
		&\frac{\sqrt{1+M_1(T_{b(\hat{\tau}_{1,j})})}}{\sigma}(\frac{\mu_1+\mu_i}{2}-\frac{S_1(T_{b(\hat{\tau}_{1,j})})}{1+M_1(T_{b(\hat{\tau}_{1,j})})})\nonumber\\
		&=\frac{\sqrt{1+M_1(T_{b(\hat{\tau}_{1,j})})}}{\sigma}(\frac{\mu_1+\mu_i}{2}-\mu_1+\mu_1-\frac{\mu_1M_1(T_{b(\hat{\tau}_{1,j})})}{1+M_1(T_{b(\hat{\tau}_{1,j})})}-\frac{S_1(T_{b(\hat{\tau}_{1,j})})-\mu_1M_1(T_{b(\hat{\tau}_{1,j})})}{1+M_1(T_{b(\hat{\tau}_{1,j})})})\nonumber\\
		&=\frac{\mu_1M_1(T_{b(\hat{\tau}_{1,j})})-S_1(T_{b(\hat{\tau}_{1,j})})}{\sigma\sqrt{1+M_1(T_{b(\hat{\tau}_{1,j})})}}+\frac{\sqrt{1+M_1(T_{b(\hat{\tau}_{1,j})})}}{\sigma}(-\frac{\Delta_i}{2}+\frac{\mu_1}{1+M_1(T_{b(\hat{\tau}_{1,j})})}).\label{ee14}
	\end{align}
	Note here that $\frac{\mu_1}{1+M_1(T_{b(\hat{\tau}_{1,j})})}\leq\frac{\Delta_i^2}{512\sigma^2\log(T)}\leq\frac{\Delta_i}{4}$ since $M_1(T_{b(\hat{\tau}_{1,j})})> 512\sigma^2\frac{\log(T)}{\Delta_i^2}$. Then \eqref{ee14} leads to
	\begin{align}
		&\frac{\sqrt{1+M_1(T_{b(\hat{\tau}_{1,j})})}}{\sigma}(\frac{\mu_1+\mu_i}{2}-\frac{S_1(T_{b(\hat{\tau}_{1,j})})}{1+M_1(T_{b(\hat{\tau}_{1,j})})})\nonumber\\
		&\leq\frac{\mu_1M_1(T_{b(\hat{\tau}_{1,j})})-S_1(T_{b(\hat{\tau}_{1,j})})}{\sigma\sqrt{1+M_1(T_{b(\hat{\tau}_{1,j})})}}-\frac{\Delta_i\sqrt{1+M_1(T_{b(\hat{\tau}_{1,j})})}}{4\sigma}\nonumber\\
		&\leq\frac{\mu_1M_1(T_{b(\hat{\tau}_{1,j})})-S_1(T_{b(\hat{\tau}_{1,j})})}{\sigma\sqrt{1+M_1(T_{b(\hat{\tau}_{1,j})})}}-4\sqrt{2\log(T)}. \label{ee15}
	\end{align}
	where the last inequality is the result of $M_1(T_{b(\hat{\tau}_{1,j})})> 512\sigma^2\frac{\log(T)}{\Delta_i^2}$. The overall analysis shows that if $x\geq 1/T$:
	\begin{align}
		&\Prob\Bigg(\hat{\tau}_{1,j}<\infty,\frac{\sqrt{1+M_1(T_{b(\hat{\tau}_{1,j})})}}{\sigma}\Big(\frac{\mu_1+\mu_i}{2}-\frac{S_1(T_{b(\hat{\tau}_{1,j})})}{1+M_1(T_{b(\hat{\tau}_{1,j})})}\Big)>Q^{-1}(\frac{1}{1+\sqrt{x}})\Bigg)\nonumber\\
		&\leq \Prob\Bigg(\hat{\tau}_{1,j}<\infty,\frac{\mu_1M_1(T_{b(\hat{\tau}_{1,j})})-S_1(T_{b(\hat{\tau}_{1,j})})}{\sqrt{1+M_1(T_{b(\hat{\tau}_{1,j})})}}>4\sigma\sqrt{2\log(T)}+\sigma Q^{-1}(\frac{1}{1+\sqrt{x}})\Bigg)\label{ee16}\\
		&\leq \exp\Big(-\frac{2\sigma^2(4\sqrt{2\log(T)}+Q^{-1}(\frac{1}{1+\sqrt{x}}))^2}{\alpha}\Big)\label{ee17}\\
		&\leq \exp\Big(-\frac{16\sigma^2\log(T)}{\alpha}\Big)\exp\Big(-\frac{2\sigma^2(2\sqrt{2\log(T)}+Q^{-1}(\frac{1}{1+\sqrt{x}}))^2}{\alpha}\Big)\label{ee18}\\
		&\leq \frac{1}{T} \label{ee19}
	\end{align}
	\eqref{ee16} follows from \eqref{ee15}. Lemma \ref{lem3} and \eqref{ee20} lead to \eqref{ee17}. Similarly \eqref{ee18} is due to \eqref{ee20} and the fact that $(a+b)^2\geq a^2+b^2$ for any non-negative $a$ and $b$. Finally, \eqref{ee19} follows from the fact that $\alpha\leq\frac{5\sigma^2}{4}$. However, for big $x$ values, we can provide a tighter upper bound. Firstly, $Q^{-1}(\frac{1}{1+\sqrt{x}})\geq \sqrt{\frac{2}{3}\log(x)}$ if $x\geq x_0$ for some $x_0\geq 4$ by \eqref{noneg}. So for $x\geq x_0$
	\begin{equation*}
		\exp\Big(-\frac{2\sigma^2(2\sqrt{2\log(T)}+Q^{-1}(\frac{1}{1+\sqrt{x}}))^2}{\alpha}\Big)\leq\exp(-\frac{4\sigma^2\log(x)}{3\alpha})\leq \frac{1}{x^{16/15}}
	\end{equation*}
	where we used the fact that $\alpha\leq\frac{5\sigma^2}{4}$. This inequality shows that
	\begin{equation}
		\Prob\Bigg(\hat{\tau}_{1,j}<\infty,\frac{\sqrt{1+M_1(T_{b(\hat{\tau}_{1,j})})}}{\sigma}\Big(\frac{\mu_1+\mu_i}{2}-\frac{S_1(T_{b(\hat{\tau}_{1,j})})}{1+M_1(T_{b(\hat{\tau}_{1,j})})}\Big)>Q^{-1}(\frac{1}{1+\sqrt{x}})\Bigg)\leq\frac{1}{Tx^{16/15}}\label{ee21}
	\end{equation}
	if $x\geq x_0$ by \eqref{ee18}. Overall, if we plug \eqref{ee19} and \eqref{ee21} back into \eqref{ee12}, we have
	\begin{equation*}
		\Ex\Bigg[\1(\hat{\tau}_{1,j}<\infty)\Bigg(\frac{1}{Q\Bigg(\frac{\sqrt{1+M_1(T_{b(\hat{\tau}_{1,j})})}}{\sigma}\Big(\frac{\mu_1+\mu_i}{2}-\frac{S_1(T_{b(\hat{\tau}_{1,j})})}{1+M_1(T_{b(\hat{\tau}_{1,j})})}\Big)
			\Bigg)}-1\Bigg)^2\Bigg]\leq \frac{1+x_0}{T}+\int_{x=x_0}^{\infty}\frac{1}{Tx^{16/15}},
	\end{equation*}
	which finishes the proof since $x_0$ is an absolute constant.
\end{proof}

\section{Proof of Theorem \ref{thm1}}\label{realpthm1}

First of all, \eqref{r4} is the immediate result of \eqref{r9} and \eqref{r3}. To prove \eqref{r6}, note that the regret contribution from the arms with $\Delta_i\leq \sigma\sqrt{\frac{\alpha K\log(T)}{T}}$ in the first $T$ rounds can not exceed $\sigma\sqrt{\alpha KT\log(T)}$. As for any arm with $\Delta_i> \sigma\sqrt{\frac{\alpha K\log(T)}{T}}$, by \eqref{r3} $\Delta_i\Ex[N_i(T)]\leq \frac{C_1\sigma\sqrt{\alpha T\log(T)}}{\sqrt{K}}$, which leads to a maximum regret contribution of $C_1\sigma\sqrt{\alpha KT\log(T)}$ from these arms. As a result, $\Ex[R(T)]\leq (1+C_1)\sigma\sqrt{\alpha KT\log(T)}$ and this proves \eqref{r6}.

We will now prove \eqref{r3}. Lets pick any $i\geq2$. We start the proof by first dividing $N_i(T)$ into smaller terms as follows:
	\begin{align}
		&\Ex[N_i(T)]=\Ex[\sum_{t=1}^{T}\1(A_t=i,\theta_i(t)>\frac{\mu_1+\mu_i}{2})]+\Ex[\sum_{t=1}^{T}\1(A_t=i,\theta_i(t)\leq\frac{\mu_1+\mu_i}{2})]\nonumber\\
		&\leq\Ex[\sum_{t=1}^{T}\1(A_t=i, \theta_i(t)>\frac{\mu_1+\mu_i}{2}, M_i(T_{b(t)})\geq 32\sigma^2\frac{\log(T)}{\Delta_i^2})]\label{p1}\\
		&+\Ex[\sum_{t=1}^{T}\1(A_t=i, t=C_{e,k}\text{ for some }k,\theta_i(t)>\frac{\mu_1+\mu_i}{2}, M_i(T_{b(t)})< 32\sigma^2\frac{\log(T)}{\Delta_i^2})]\label{p2}\\
		&+\Ex[\sum_{t=1}^{T}\1(A_t=i, t\neq C_{e,k}\text{ for all }k,\theta_i(t)>\frac{\mu_1+\mu_i}{2}, M_i(T_{b(t)})< 32\sigma^2\frac{\log(T)}{\Delta_i^2})]\label{p3}\\
		&+1+\Ex[\sum_{t=2}^{T}\1(A_t=i,\theta_i(t)\leq\frac{\mu_1+\mu_i}{2}, M_1(T_{b(t)})\geq 32\sigma^2\frac{\log(T)}{\Delta_i^2})]\label{p4}\\
		&+\Ex[\sum_{t=2}^{T}\1(A_t=i, A_{t-1}=1, \theta_i(t)\leq\frac{\mu_1+\mu_i}{2}, M_1(T_{b(t)})< 32\sigma^2\frac{\log(T)}{\Delta_i^2})]\label{p5}\\
		&+\Ex[\sum_{t=2}^{T}\1(A_t=i, A_{t-1}\neq1, \theta_i(t)\leq\frac{\mu_1+\mu_i}{2}, M_1(T_{b(t)})< 32\sigma^2\frac{\log(T)}{\Delta_i^2})]\label{p6}
	\end{align}
	The division of these terms rely on the following conditions: $\theta_i(t)$ is above or below a certain level, $M_i(T_{b(t)})$ and $M_1(T_{b(t)})$ are above or below a certain level, $t$ is a cycle end point or not, and finally, whether the first action is chosen at $t-1$ or not. Here we will bound each expectation term individually. 
	
	First of all, the terms with $M_i(T_{b(t)})$ or $M_1(T_{b(t)})$ being bigger than a certain level account for the estimation error and will be bounded by  constants. We know that by \eqref{r22} of Proposition \ref{prop1} that 
	\begin{equation}
		\Ex[\sum_{t=1}^{T}\1(A_t=i, \theta_i(t)>\frac{\mu_1+\mu_i}{2}, M_i(T_{b(t)})\geq 32\sigma^2\frac{\log(T)}{\Delta_i^2})]\leq 2.\label{e23}
	\end{equation}
	Now note that if $A_t=i$ and $\theta_i(t)\leq\frac{\mu_1+\mu_i}{2}$, then $\theta_1(t)\leq\frac{\mu_1+\mu_i}{2}$. That means
	\begin{align}
		&\Ex[\sum_{t=2}^{T}\1(A_t=i,\theta_i(t)\leq\frac{\mu_1+\mu_i}{2}, M_1(T_{b(t)})\geq 32\sigma^2\frac{\log(T)}{\Delta_i^2})]\nonumber\\
		&\leq \Ex[\sum_{t=2}^{T}\1(\theta_1(t)\leq\frac{\mu_1+\mu_i}{2}, M_1(T_{b(t)})\geq 32\sigma^2\frac{\log(T)}{\Delta_i^2})]\leq 2\label{e24}
	\end{align}
	where the last inequality follows from \eqref{r11} of Proposition \ref{prop1}.
	
	Now we bound the remaining terms. We first note that the condition $\{A_t=i, t=C_{e,k}\text{ for some }k\}$ signifies a cycle where the $i^{th}$ arm has been played at the end. Considering that the cycle count from the last batch can not increase more than its $\alpha$ multiple plus one by Algorithm \ref{algo}, we know 
	\begin{equation}
		\Ex[\sum_{t=1}^{T}\1(A_t=i, t=C_{e,k}\text{ for some }k,\theta_i(t)>\frac{\mu_1+\mu_i}{2}, M_i(T_{b(t)})< 32\sigma^2\frac{\log(T)}{\Delta_i^2})]\leq 32\alpha\sigma^2\frac{\log(T)}{\Delta_i^2}+1\label{e25}
	\end{equation} 
	due to the $\{M_i(T_{b(t)})< 32\sigma^2\frac{\log(T)}{\Delta_i^2}\}$ condition restricting the number of times we can count a unique cycle. Similarly the condition $\{A_t=i, A_{t-1}=1\}$ means that $t-1$ is either a cycle beginning or end point, and again by $M_1(T_{b(t)})< 32\sigma^2\frac{\log(T)}{\Delta_i^2}$ limiting the number of times we can count a unique cycle with the first action we have
	\begin{equation}
		\Ex[\sum_{t=2}^{T}\1(A_t=i, A_{t-1}=1, \theta_i(t)\leq\frac{\mu_1+\mu_i}{2}, M_1(T_{b(t)})< 32\sigma^2\frac{\log(T)}{\Delta_i^2})]\leq 32\alpha\sigma^2\frac{\log(T)}{\Delta_i^2}+1.\label{e26}
	\end{equation} 
	
	Finally the only summands we did not bound are in \eqref{p3} and \eqref{p6}. We will start with the harder \eqref{p3} one and use the analysis there to bound \eqref{p6} at the end. Let $\mathcal{\hat{H}}_t=\{\Hp_t,\theta_{A_1}(1),\theta_{A_2}(2),...,\theta_{A_t}(t)\}$ and define the following stopping times for $\{\mathcal{\hat{H}}_t\}$
	\begin{equation*}
		\tau_{b,k} =
		\left\{
		\begin{array}{ll}
			\min\{t\in\mathbb{Z}^+| A_t=i, t\neq C_{e,k}\text{ for all }k,\theta_i(t)>\frac{\mu_1+\mu_i}{2}\}  & \mbox{if } k=1 \\
			\min\{t\in\mathbb{Z}^+|A_t=i, t\neq C_{e,k}\text{ for all }k,\theta_i(t)>\frac{\mu_1+\mu_i}{2}, t>\tau_{e,k-1}\} & \mbox{if } k>1
		\end{array}
		\right.
	\end{equation*}
	and
	\begin{equation*}
		\tau_{e,k}=\min\{t\in\mathbb{Z}^+| A_t\neq i\text{ or }\theta_i(t)\leq \frac{\mu_1+\mu_i}{2}\text{ such that } t>\tau_{b,k}\}.
	\end{equation*}
	Note that if any of these $\min$ operators are over an empty set, then the random variable is set to infinity. By the definitions of $\tau_{b,k}$ and $\tau_{e,k}$, it is easy to see that $\1(A_t=i, t\neq C_{e,k}\text{ for all }k,\theta_i(t)>\frac{\mu_1+\mu_i}{2}, M_i(T_{b(t)})< 32\sigma^2\frac{\log(T)}{\Delta_i^2})=1$ only if $\tau_{b,k}\leq t<\tau_{e,k}$ for some $k$. This observations suggests that it is enough to only consider the intervals of $[\tau_{b,k},\tau_{e,k}-1]$ while summing over the elements of $1(A_t=i, t\neq C_{e,k}\text{ for all }k,\theta_i(t)>\frac{\mu_1+\mu_i}{2}, M_i(T_{b(t)})< 32\sigma^2\frac{\log(T)}{\Delta_i^2})$ in \eqref{p3}. However, we can only consider an interval of $[\tau_{b,k},\tau_{e,k}-1]$ if $\tau_{b,k}<\infty$. That means we have the following bound
	
	\begin{align}
		&\sum_{t=1}^{T}\1(A_t=i, t\neq C_{e,k}\text{ for all }k,\theta_i(t)>\frac{\mu_1+\mu_i}{2}, M_i(T_{b(t)})< 32\sigma^2\frac{\log(T)}{\Delta_i^2})\nonumber\\
		&=\sum_{t=1}^{T}\1(A_t=i, t\neq C_{e,k}\text{ for all }k,\theta_i(t)>\frac{\mu_1+\mu_i}{2}, M_i(T_{b(t)})< 32\sigma^2\frac{\log(T)}{\Delta_i^2}, M_1(t)\leq T)\label{e27}\\
		&\leq\sum_{k=1}^{\infty}\1(\tau_{b,k}<\infty)\nonumber\\
		&\qquad\sum_{t=\tau_{b,k}}^{\tau_{e,k}-1}\1(A_t=i, t\neq C_{e,k}\text{ for all }k,\theta_i(t)>\frac{\mu_1+\mu_i}{2}, M_i(T_{b(t)})< 32\sigma^2\frac{\log(T)}{\Delta_i^2}, M_1(t)\leq T)\label{e28}\\
		&\leq\sum_{k=1}^{\infty}\1(\tau_{b,k}<\infty, M_i(T_{b(\tau_{b,k})})< 32\sigma^2\frac{\log(T)}{\Delta_i^2}, M_1(\tau_{b,k})\leq T)(\tau_{e,k}-\tau_{b,k}),\label{e29}
	\end{align}
	where \eqref{e27} follows from the fact that $M_1(t)\leq T$ is satisfied for any $t\leq T$. Earlier discussion leads to \eqref{e28}. Finally, note that the time interval $[\tau_{b,k},\tau_{e,k}-1]$, where only the $i^{th}$ action is played, is inside a single cycle, so $M_i(T_{b(t)})$ and $M_1(t)$ stay the same for any $t\in[\tau_{b,k},\tau_{e,k}-1]$. This observation and ignoring the condition $\{A_t=i, t\neq C_{e,k}\text{ for all }k,\theta_i(t)>\frac{\mu_1+\mu_i}{2}\}$ inside the indicator functions lead to \eqref{e29}. As the for expectation of the summands in \eqref{e29}: 
	\begin{align}
		&\Ex[\1(\tau_{b,k}<\infty, M_i(T_{b(\tau_{b,k})})< 32\sigma^2\frac{\log(T)}{\Delta_i^2}, M_1(\tau_{b,k})\leq T)(\tau_{e,k}-\tau_{b,k})]\nonumber\\
		&=\Ex[\Ex[\1(\tau_{b,k}<\infty, M_i(T_{b(\tau_{b,k})})< 32\sigma^2\frac{\log(T)}{\Delta_i^2}, M_1(\tau_{b,k})\leq T)(\tau_{e,k}-\tau_{b,k})|\mathcal{\hat{H}}_{\tau_{b,k}}]]\nonumber\\
		&=\Ex[\1(\tau_{b,k}<\infty, M_i(T_{b(\tau_{b,k})})< 32\sigma^2\frac{\log(T)}{\Delta_i^2}, M_1(\tau_{b,k})\leq T)\Ex[\tau_{e,k}-\tau_{b,k}|\mathcal{\hat{H}}_{\tau_{b,k}}]]\label{e30},
	\end{align}
	The last inequality is the result of the indicator function inside the expectation being measurable with respect to $\mathcal{\hat{H}}_{\tau_{b,k}}$. If $\tau_{b,k}<\infty$, then conditioned on $\mathcal{\hat{H}}_{\tau_{b,k}}$ we know that $\tau_{e,k}-\tau_{b,k}-1$ is a geometric random variable for failures, where the the success probability is $1-\Prob(A_{\tau_{b,k}+1}=i, \theta_i(\tau_{b,k}+1)>\frac{\mu_1+\mu_i}{2}|\mathcal{\hat{H}}_{\tau_{b,k}})$. The reason is that since $\tau_{b,k}<\infty$, we know that $\tau_{e,k}-\tau_{b,k}\geq 1$ and by the definitions of $\tau_{b,k}$ and $\tau_{e,k}$ $[\tau_{b,k},\tau_{e,k}]$ defines a time interval in a single cycle like we mentioned earlier. These observations show that the sampling process remains the same throughout $[\tau_{b,k},\tau_{e,k}]$, and conditioned on $\mathcal{\hat{H}}_{\tau_{b,k}}$ $[\tau_{b,k}+1,\tau_{e,k}-1]$ is a period of failures if we were to define success as $A_i(t)\neq i$ or $\theta_i(t)\leq \frac{\mu_1+\mu_i}{2}$. The overall analysis proves the following set of inequalities
	\begin{align}
		&\1(\tau_{b,k}<\infty)\Ex[\tau_{e,k}-\tau_{b,k}|\mathcal{\hat{H}}_{\tau_{b,k}}]=\1(\tau_{b,k}<\infty)\frac{1}{1-\Prob(A_{\tau_{b,k}+1}=i, \theta_i(\tau_{b,k}+1)>\frac{\mu_1+\mu_i}{2}|\mathcal{\hat{H}}_{\tau_{b,k}})}\nonumber\\
		&\leq\1(\tau_{b,k}<\infty)\frac{1}{\Prob( \theta_i(\tau_{b,k}+1)\leq\frac{\mu_1+\mu_i}{2}|\mathcal{\hat{H}}_{\tau_{b,k}})}\nonumber
	\end{align}
	almost surely. Putting this inequality inside \eqref{e30} and summing the elements like in \eqref{e29} shows that
	\begin{align}
		&\Ex[\sum_{t=1}^{T}\1(A_t=i, t\neq C_{e,k}\text{ for all }k,\theta_i(t)>\frac{\mu_1+\mu_i}{2}, M_i(T_{b(t)})< 32\sigma^2\frac{\log(T)}{\Delta_i^2})]\nonumber\\
		&\leq\Ex[\sum_{k=1}^{\infty}\1(\tau_{b,k}<\infty, M_i(T_{b(\tau_{b,k})})< 32\sigma^2\frac{\log(T)}{\Delta_i^2}, M_1(\tau_{b,k})\leq T)\frac{1}{\Prob( \theta_i(\tau_{b,k}+1)\leq\frac{\mu_1+\mu_i}{2}|\mathcal{\hat{H}}_{\tau_{b,k}})}]\nonumber\\
		&=\Ex[\sum_{k=1}^{\infty}\1(\tau_{b,k}<\infty, M_i(T_{b(\tau_{b,k})})< 32\sigma^2\frac{\log(T)}{\Delta_i^2}, M_1(\tau_{b,k})\leq T, A_{\tau_{e,k}}=i)\frac{1}{\Prob( \theta_i(\tau_{b,k}+1)\leq\frac{\mu_1+\mu_i}{2}|\mathcal{\hat{H}}_{\tau_{b,k}})}]\label{e31}\\
		&+\Ex[\sum_{k=1}^{\infty}\1(\tau_{b,k}<\infty, M_i(T_{b(\tau_{b,k})})< 32\sigma^2\frac{\log(T)}{\Delta_i^2}, M_1(\tau_{b,k})\leq T, A_{\tau_{e,k}}\neq i)\frac{1}{\Prob( \theta_i(\tau_{b,k}+1)\leq\frac{\mu_1+\mu_i}{2}|\mathcal{\hat{H}}_{\tau_{b,k}})}]\label{e35}
	\end{align}
	Note by the earlier analysis we know that on $\{\tau_{b,k}<\infty\}$ $\tau_{e,k}$ is almost surely finite. Then the last equality follows from dividing the terms according to $A_{\tau_{e,k}}=i$ or not. We will first bound the summand in \eqref{e31}. To that end, we now analyze $\1(\tau_{b,k}<\infty)\Prob(A_{\tau_{e,k}}=1|\mathcal{\hat{H}}_{\tau_{b,k}})$. Note that by the earlier analysis we know the sampling distributions remains the same throughout $[\tau_{b,k},\tau_{e,k}]$, which leads to
	\begin{align}
		&\Prob(A_{\tau_{e,k}}=1|\mathcal{\hat{H}}_{\tau_{b,k}})=\frac{\Prob(A_{\tau_{b,k}+1}=1|\mathcal{\hat{H}}_{\tau_{b,k}})}{1-\Prob(A_{\tau_{b,k}+1}=i, \theta_i(\tau_{b,k}+1)>\frac{\mu_1+\mu_i}{2}|\mathcal{\hat{H}}_{\tau_{b,k}})}\label{e32}\\
		&\geq\frac{\Prob(\theta_1(\tau_{b,k}+1)>\frac{\mu_1+\mu_i}{2}, \theta_j(\tau_{b,k}+1)\leq\frac{\mu_1+\mu_i}{2}\text{ for all } j\neq 1|\mathcal{\hat{H}}_{\tau_{b,k}})}{1-\Prob(A_{\tau_{b,k}+1}=i, \theta_i(\tau_{b,k}+1)>\frac{\mu_1+\mu_i}{2}|\mathcal{\hat{H}}_{\tau_{b,k}})}\label{e33}\\
		&=\frac{\Prob(\theta_1(\tau_{b,k}+1)>\frac{\mu_1+\mu_i}{2}|\mathcal{\hat{H}}_{\tau_{b,k}})\Prob( \theta_j(\tau_{b,k}+1)\leq\frac{\mu_1+\mu_i}{2}\text{ for all } j\neq 1|\mathcal{\hat{H}}_{\tau_{b,k}})}{1-\Prob(A_{\tau_{b,k}+1}=i, \theta_i(\tau_{b,k}+1)>\frac{\mu_1+\mu_i}{2}|\mathcal{\hat{H}}_{\tau_{b,k}})}\label{e34}
	\end{align}
	on $\{\tau_{b,k}<\infty\}$. Note that since $\tau_{b,k}<\infty$, all the conditional probabilities stated here are almost surely positive. Here \eqref{e32} trivially follows from the definition of success and failure of the geometric random variable we have defined earlier, i.e. $\tau_{e,k}-\tau_{b,k}-1$. \eqref{e33} is the result of the action selection process where we know that for the first action to be chosen the sample $\theta_1$ has to be at least as big as the other samples. Finally, conditioned on $\mathcal{\hat{H}}_{\tau_{b,k}}$, $\{\theta_j(\tau_{b,k}+1)\}_{j=1}^K$ are independent, which results in \eqref{e34}. On the other hand
	\begin{align}
		&\Prob(A_{\tau_{e,k}}=i|\mathcal{\hat{H}}_{\tau_{b,k}})=\frac{\Prob(A_{\tau_{b,k}+1}=i,\theta_i(\tau_{b,k}+1)\leq\frac{\mu_1+\mu_i}{2}|\mathcal{\hat{H}}_{\tau_{b,k}})}{1-\Prob(A_{\tau_{b,k}+1}=i, \theta_i(\tau_{b,k}+1)>\frac{\mu_1+\mu_i}{2}|\mathcal{\hat{H}}_{\tau_{b,k}})}\label{e36}\\
		&\leq\frac{\Prob(\theta_j(\tau_{b,k}+1)\leq\frac{\mu_1+\mu_i}{2}\text{ for all }j|\mathcal{\hat{H}}_{\tau_{b,k}})}{1-\Prob(A_{\tau_{b,k}+1}=i, \theta_i(\tau_{b,k}+1)>\frac{\mu_1+\mu_i}{2}|\mathcal{\hat{H}}_{\tau_{b,k}})}\label{e37}\\
		&=\frac{\Prob(\theta_1(\tau_{b,k}+1)\leq\frac{\mu_1+\mu_i}{2}|\mathcal{\hat{H}}_{\tau_{b,k}})\Prob( \theta_j(\tau_{b,k}+1)\leq\frac{\mu_1+\mu_i}{2}\text{ for all } j\neq 1|\mathcal{\hat{H}}_{\tau_{b,k}})}{1-\Prob(A_{\tau_{b,k}+1}=i, \theta_i(\tau_{b,k}+1)>\frac{\mu_1+\mu_i}{2}|\mathcal{\hat{H}}_{\tau_{b,k}})}\label{e38}
	\end{align}
	on $\{\tau_{b,k}<\infty\}$. Note that \eqref{e36} follows from the fact that $A_{\tau_{e,k}}=i$ only if $\theta_i(\tau_{e,k})\leq\frac{\mu_1+\mu_i}{2}$ by the definition of $\tau_{e,k}$. Considering that $\{A_{\tau_{b,k}+1}=i,\theta_i(\tau_{b,k}+1)\leq\frac{\mu_1+\mu_i}{2}\}$ means $\theta_j(\tau_{b,k}+1)\leq\frac{\mu_1+\mu_i}{2}\text{ for all }j$, we have \eqref{e37}. \eqref{e38} is the result of the conditional independence. Combining \eqref{e34} and \eqref{e38} leads to
	\begin{equation}
		\Prob(A_{\tau_{e,k}}=i|\mathcal{\hat{H}}_{\tau_{b,k}})\leq\frac{\Prob(\theta_1(\tau_{b,k}+1)\leq\frac{\mu_1+\mu_i}{2}|\mathcal{\hat{H}}_{\tau_{b,k}})}{\Prob(\theta_1(\tau_{b,k}+1)>\frac{\mu_1+\mu_i}{2}|\mathcal{\hat{H}}_{\tau_{b,k}})}\Prob(A_{\tau_{e,k}}=1|\mathcal{\hat{H}}_{\tau_{b,k}}).\label{e43}
	\end{equation}
	on $\{\tau_{b,k}<\infty\}$. As a result
	\begin{align}
		&\Ex[\1(\tau_{b,k}<\infty, M_i(T_{b(\tau_{b,k})})< 32\sigma^2\frac{\log(T)}{\Delta_i^2}, M_1(\tau_{b,k})\leq T, A_{\tau_{e,k}}=i)\frac{1}{\Prob( \theta_i(\tau_{b,k}+1)\leq\frac{\mu_1+\mu_i}{2}|\mathcal{\hat{H}}_{\tau_{b,k}})}]\nonumber\\
		&=\Ex[\Ex[\1(\tau_{b,k}<\infty, M_i(T_{b(\tau_{b,k})})< 32\sigma^2\frac{\log(T)}{\Delta_i^2}, M_1(\tau_{b,k})\leq T)\frac{1}{\Prob( \theta_i(\tau_{b,k}+1)\leq\frac{\mu_1+\mu_i}{2}|\mathcal{\hat{H}}_{\tau_{b,k}})}\nonumber\\
		&\qquad\qquad\qquad\qquad\qquad\qquad\qquad\qquad\qquad\qquad\qquad\1(A_{\tau_{e,k}}=i)|\mathcal{\hat{H}}_{\tau_{b,k}}]]\nonumber\\
		&=\Ex[\1(\tau_{b,k}<\infty, M_i(T_{b(\tau_{b,k})})< 32\sigma^2\frac{\log(T)}{\Delta_i^2}, M_1(\tau_{b,k})\leq T)\frac{1}{\Prob( \theta_i(\tau_{b,k}+1)\leq\frac{\mu_1+\mu_i}{2}|\mathcal{\hat{H}}_{\tau_{b,k}})}\nonumber\\
		&\qquad\qquad\qquad\qquad\qquad\qquad\qquad\qquad\qquad\qquad\qquad\Prob(A_{\tau_{e,k}}=i|\mathcal{\hat{H}}_{\tau_{b,k}})]\label{e39}\\
		&\leq\Ex[\1(\tau_{b,k}<\infty, M_i(T_{b(\tau_{b,k})})< 32\sigma^2\frac{\log(T)}{\Delta_i^2}, M_1(\tau_{b,k})\leq T)\frac{1}{\Prob( \theta_i(\tau_{b,k}+1)\leq\frac{\mu_1+\mu_i}{2}|\mathcal{\hat{H}}_{\tau_{b,k}})}\nonumber\\
		&\qquad\qquad\qquad\qquad\frac{\Prob(\theta_1(\tau_{b,k}+1)\leq\frac{\mu_1+\mu_i}{2}|\mathcal{\hat{H}}_{\tau_{b,k}})}{\Prob(\theta_1(\tau_{b,k}+1)>\frac{\mu_1+\mu_i}{2}|\mathcal{\hat{H}}_{\tau_{b,k}})}\Prob(A_{\tau_{e,k}}=1|\mathcal{\hat{H}}_{\tau_{b,k}})]\label{e40}\\
		&=\Ex[\1(\tau_{b,k}<\infty, M_i(T_{b(\tau_{b,k})})< 32\sigma^2\frac{\log(T)}{\Delta_i^2}, M_1(\tau_{b,k})\leq T)\frac{1}{\Prob( \theta_i(\tau_{b,k}+1)\leq\frac{\mu_1+\mu_i}{2}|\mathcal{\hat{H}}_{\tau_{b,k}})}\nonumber\\
		&\qquad\qquad\qquad\qquad\frac{\Prob(\theta_1(\tau_{b,k}+1)\leq\frac{\mu_1+\mu_i}{2}|\mathcal{\hat{H}}_{\tau_{b,k}})}{\Prob(\theta_1(\tau_{b,k}+1)>\frac{\mu_1+\mu_i}{2}|\mathcal{\hat{H}}_{\tau_{b,k}})}\1(A_{\tau_{e,k}}=1)]\nonumber\\
		&\leq \Ex[\1(\tau_{b,k}<\infty, A_{\tau_{e,k}}=1,M_i(T_{b(\tau_{b,k})})< 32\sigma^2\frac{\log(T)}{\Delta_i^2}, M_1(\tau_{b,k})\leq T)\frac{1}{\Prob^2( \theta_i(\tau_{b,k}+1)\leq\frac{\mu_1+\mu_i}{2}|\mathcal{\hat{H}}_{\tau_{b,k}})}]\nonumber\\
		&+\Ex[\1(\tau_{b,k}<\infty, A_{\tau_{e,k}}=1,M_i(T_{b(\tau_{b,k})})< 32\sigma^2\frac{\log(T)}{\Delta_i^2}, M_1(\tau_{b,k})\leq T)\frac{\Prob^2(\theta_1(\tau_{b,k}+1)\leq\frac{\mu_1+\mu_i}{2}|\mathcal{\hat{H}}_{\tau_{b,k}})}{\Prob^2(\theta_1(\tau_{b,k}+1)>\frac{\mu_1+\mu_i}{2}|\mathcal{\hat{H}}_{\tau_{b,k}})}]\label{e41}.
	\end{align}
	\eqref{e39} follows from the measurability of the indicator function  and the inverse probability term with respect to $\mathcal{\hat{H}}_{\tau_{b,k}}$. \eqref{e43} leads to \eqref{e40}. Finally, the last inequality follows from the fact that $2\sqrt{a\times b}\leq a+b$ for any non-negative $a$ and $b$. Here \eqref{e41} shows that
	\begin{align}
		&\Ex[\sum_{k=1}^{\infty}\1(\tau_{b,k}<\infty, M_i(T_{b(\tau_{b,k})})< 32\sigma^2\frac{\log(T)}{\Delta_i^2}, M_1(\tau_{b,k})\leq T, A_{\tau_{e,k}}=i)\frac{1}{\Prob( \theta_i(\tau_{b,k}+1)\leq\frac{\mu_1+\mu_i}{2}|\mathcal{\hat{H}}_{\tau_{b,k}})}]\nonumber\\
		&\leq\Ex[\sum_{k=1}^{\infty}\1(\tau_{b,k}<\infty, A_{\tau_{e,k}}=1,M_i(T_{b(\tau_{b,k})})< 32\sigma^2\frac{\log(T)}{\Delta_i^2})\frac{1}{\Prob^2( \theta_i(\tau_{b,k}+1)\leq\frac{\mu_1+\mu_i}{2}|\mathcal{\hat{H}}_{\tau_{b,k}})}]\nonumber\\
		&+\Ex[\sum_{k=1}^{\infty}\1(\tau_{b,k}<\infty, A_{\tau_{e,k}}=1, M_1(\tau_{b,k})\leq T)\frac{\Prob^2(\theta_1(\tau_{b,k}+1)\leq\frac{\mu_1+\mu_i}{2}|\mathcal{\hat{H}}_{\tau_{b,k}})}{\Prob^2(\theta_1(\tau_{b,k}+1)>\frac{\mu_1+\mu_i}{2}|\mathcal{\hat{H}}_{\tau_{b,k}})}]\label{e46}
	\end{align}
	 Here we eliminated one condition from each indicator function in the last inequality. However, we know by the action selection process of Thompson sampling and $b(\tau_{b,k}+1)=b(\tau_{b,k})$ equality due to $\tau_{b,k}$ and $\tau_{b,k}+1$ being in the same cycle that
	\begin{equation}
		\Prob(\theta_1(\tau_{b,k}+1)>\frac{\mu_1+\mu_i}{2}|\mathcal{\hat{H}}_{\tau_{b,k}})=Q\Bigg(\frac{\sqrt{1+M_1(T_{b(\tau_{b,k})})}}{\sigma}\Big(\frac{\mu_1+\mu_i}{2}-\frac{S_1(T_{b(\tau_{b,k})})}{1+M_1(T_{b(\tau_{b,k})})}\Big)
		\Bigg)\label{e44}
	\end{equation}
	and
	\begin{equation}
		\Prob(\theta_i(\tau_{b,k}+1)\leq\frac{\mu_1+\mu_i}{2}|\mathcal{\hat{H}}_{\tau_{b,k}})=Q\Bigg(\frac{\sqrt{1+M_i(T_{b(\tau_{b,k})})}}{\sigma}\Big(\frac{S_i(T_{b(\tau_{b,k})})}{1+M_i(T_{b(\tau_{b,k})})}-\frac{\mu_1+\mu_i}{2}\Big)
		\Bigg)\label{e45}
	\end{equation}
	on $\{\tau_{b,k}<\infty\}$. Considering \eqref{e44} and \eqref{e45}, we see can view \eqref{e46} as  
	\begin{align}
		&\Ex[\sum_{k=1}^{\infty}\1(\tau_{b,k}<\infty, M_i(T_{b(\tau_{b,k})})< 32\sigma^2\frac{\log(T)}{\Delta_i^2}, M_1(\tau_{b,k})\leq T, A_{\tau_{e,k}}=i)\frac{1}{\Prob( \theta_i(\tau_{b,k}+1)\leq\frac{\mu_1+\mu_i}{2}|\mathcal{\hat{H}}_{\tau_{b,k}})}]\nonumber\\
		&\leq\Ex[\sum_{k=1}^{\infty}\1(\tau_{b,k}<\infty, A_{\tau_{e,k}}=1,M_i(T_{b(\tau_{b,k})})< 32\sigma^2\frac{\log(T)}{\Delta_i^2})f_1(M_i(T_{b(\tau_{b,k})}),S_i(T_{b(\tau_{b,k})}))]\nonumber\\
		&+\Ex[\sum_{k=1}^{\infty}\1(\tau_{b,k}<\infty, A_{\tau_{e,k}}=1, M_1(\tau_{b,k})\leq T)f_2(M_1(T_{b(\tau_{b,k})}),S_1(T_{b(\tau_{b,k})}))]\nonumber
	\end{align}
	where $f_1$ and $f_2$ are some functions with domain $\mathbb{R}_{\geq0}\times\mathbb{R}$. If we let
	\begin{equation*}
		\hat{\tau}_{i,j} =
		\left\{
		\begin{array}{ll}
			\min\{t\in\mathbb{Z}^+| A_t=i, t=C_{b,k} \text{ or } C_{e,k}\text{ for some }k\}  & \mbox{if } j=1 \\
			\min\{t\in\mathbb{Z}^+| A_t=i, t=C_{b,k} \text{ or } C_{e,k}\text{ for some }k, t>\hat{\tau}_{i,j-1}\} & \mbox{if } j>1,
		\end{array}
		\right.
	\end{equation*}
	which denotes the time indices where we choose the $i^{th}$ arm at the beginning or at the end of a cycle, we notice that $M_i(T_{b(\tau_{b,k})})=M_i(T_{b(\hat{\tau}_{i,j})})$ and $S_i(T_{b(\tau_{b,k})})=S_i(T_{b(\hat{\tau}_{i,j})})$ for some $j$ on $\{\tau_{b,k}<\infty\}$ since $\tau_{b,k}<\infty$ means that the agent has played the $i^{th}$ action at the beginning of the cycle containing $\tau_{b,k}$. However, when we look $\1(\tau_{b,k}<\infty, A_{\tau_{e,k}}=1,M_i(T_{b(\tau_{b,k})})< 32\sigma^2\frac{\log(T)}{\Delta_i^2})f_1(M_i(T_{b(\tau_{b,k})}),S_i(T_{b(\tau_{b,k})}))$ terms, we realize that each $[\tau_{b,k},\tau_{e,k}-1]$ will belong to a different cycle if this indicator function is non-zero due to $A_{\tau_{e,k}}=1$, and the condition $\{M_i(T_{b(\tau_{b,k})})< 32\sigma^2\frac{\log(T)}{\Delta_i^2}\}$ implies that the indicator function can be non-zero only if $[\tau_{b,k},\tau_{e,k}]$ is inside the one of the first $\lceil32\alpha\sigma^2\frac{\log(T)}{\Delta_i^2}\rceil$ cycles of the $i^{th}$ arm. The overall discussion leads to the following bound:
	\begin{align}
		&\sum_{k=1}^{\infty}\1(\tau_{b,k}<\infty, A_{\tau_{e,k}}=1,M_i(T_{b(\tau_{b,k})})< 32\sigma^2\frac{\log(T)}{\Delta_i^2})f_1(M_i(T_{b(\tau_{b,k})}),S_i(T_{b(\tau_{b,k})}))\nonumber\\
		&\leq \sum_{j=1}^{\lceil32\alpha\sigma^2\frac{\log(T)}{\Delta_i^2}\rceil}\1(\hat{\tau}_{i,j}<\infty)f_1(M_i(T_{b(\tau_{i,j})}),S_i(T_{b(\tau_{i,j})}))\label{e47}
	\end{align} 
	On the other hand, on $\{\tau_{b,k}<\infty, A_{\tau_{e,k}}=1\}$, $M_1(T_{b(\tau_{b,k})})=M_1(T_{b(\hat{\tau}_{1,j})})$ and $S_1(T_{b(\tau_{b,k})})=S_1(T_{b(\hat{\tau}_{1,j})})$ for some $j$ since $\tau_{e,k}$ here is the cycle end point. Similar to the earlier analysis, for each $k$ that satisfies the $\{\tau_{b,k}<\infty, A_{\tau_{e,k}}=1, M_1(\tau_{b,k})\leq T\}$ condition, $[\tau_{b,k},\tau_{e,k}]$ will be in a distinct cycle from the first $T+1$ ones containing the first arm. This observation shows that
	\begin{align}
		&\sum_{k=1}^{\infty}\1(\tau_{b,k}<\infty, A_{\tau_{e,k}}=1, M_1(\tau_{b,k})\leq T)f_2(M_1(T_{b(\tau_{b,k})}),S_1(T_{b(\tau_{b,k})}))\nonumber\\
		&\leq\sum_{j=1}^{T+1}\1(\hat{\tau}_{1,j}<\infty)f_2(M_1(T_{b(\tau_{1,j})}),S_1(T_{b(\tau_{1,j})}))\label{e48}
	\end{align}
	In view of \eqref{e46}, \eqref{e47} and \eqref{e48} result in the following bound
	\begin{align}
		&\Ex[\sum_{k=1}^{\infty}\1(\tau_{b,k}<\infty, M_i(T_{b(\tau_{b,k})})< 32\sigma^2\frac{\log(T)}{\Delta_i^2}, M_1(\tau_{b,k})\leq T, A_{\tau_{e,k}}=i)\frac{1}{\Prob( \theta_i(\tau_{b,k}+1)\leq\frac{\mu_1+\mu_i}{2}|\mathcal{\hat{H}}_{\tau_{b,k}})}]\nonumber\\
		&\leq\Ex\Bigg[\sum_{j=1}^{\lceil32\alpha\sigma^2\frac{\log(T)}{\Delta_i^2}\rceil}\1(\hat{\tau}_{i,j}<\infty)\frac{1}{Q^2\Bigg(\frac{\sqrt{1+M_i(T_{b(\hat{\tau}_{i,j})})}}{\sigma}\Big(\frac{S_i(T_{b(\hat{\tau}_{i,j})})}{1+M_i(T_{b(\hat{\tau}_{i,j})})}-\frac{\mu_1+\mu_i}{2}\Big)
			\Bigg)}\Bigg]\nonumber\\
		&+\Ex\Bigg[\sum_{j=1}^{T+1}\1(\hat{\tau}_{1,j}<\infty)\Bigg(\frac{1}{Q\Bigg(\frac{\sqrt{1+M_1(T_{b(\hat{\tau}_{1,j})})}}{\sigma}\Big(\frac{\mu_1+\mu_i}{2}-\frac{S_1(T_{b(\hat{\tau}_{1,j})})}{1+M_1(T_{b(\hat{\tau}_{1,j})})}\Big)
			\Bigg)}-1\Bigg)^2\Bigg]\label{e49}
	\end{align}
	where we replaced $f_1$ and $f_2$ with their exact forms. Note that although we did not define $f_1$ and $f_2$ explicitly, it is easy understand their exact formulation from the earlier discussion, i.e. from the conditional probability functions in \eqref{e46} and the equalities stated in \eqref{e44} and \eqref{e45}. Here we know that the first expectation to the right-side of the inequality is upper bounded by $C(1+32\alpha\sigma^2\frac{\log(T)}{\Delta_i^2})$ by Proposition \ref{prop2}, where $C$ is an absolute constant. On the other hand, we have
	\begin{multline*}
		\Ex\Bigg[\1(\hat{\tau}_{1,j}<\infty)\Bigg(\frac{1}{Q\Bigg(\frac{\sqrt{1+M_1(T_{b(\hat{\tau}_{1,j})})}}{\sigma}\Big(\frac{\mu_1+\mu_i}{2}-\frac{S_1(T_{b(\hat{\tau}_{1,j})})}{1+M_1(T_{b(\hat{\tau}_{1,j})})}\Big)
			\Bigg)}-1\Bigg)^2\Bigg] \\
		\leq
		\left\{
		\begin{array}{ll}
			C  & \mbox{if } j\leq 1024\alpha\sigma^2\frac{\log(T)}{\Delta_i^2} \\
			C/T & \mbox{if } j>1024\alpha\sigma^2\frac{\log(T)}{\Delta_i^2},
		\end{array}
		\right.
	\end{multline*}   
	where $C$ is an absolute constant independent of the system variables. Note that the constant bound follows from Proposition \ref{prop2}, while $O(1/T)$ bound is the result of Proposition \ref{prop3}. This overall analysis and \eqref{e49} lead to
	\begin{align}
		&\Ex[\sum_{k=1}^{\infty}\1(\tau_{b,k}<\infty, M_i(T_{b(\tau_{b,k})})< 32\sigma^2\frac{\log(T)}{\Delta_i^2}, M_1(\tau_{b,k})\leq T, A_{\tau_{e,k}}=i)\frac{1}{\Prob( \theta_i(\tau_{b,k}+1)\leq\frac{\mu_1+\mu_i}{2}|\mathcal{\hat{H}}_{\tau_{b,k}})}]\nonumber\\
		&\leq C(2+1056\alpha\sigma^2\frac{\log(T)}{\Delta_i^2})\label{e50}
	\end{align}
	for an absolute constant $C$. This proof bounds the term in \eqref{e31}. However, with the analysis we have done so far, bounding the term in \eqref{e35} is almost immediate. First note that, similar to the earlier analysis, $M_i(T_{b(\tau_{b,k})})=M_i(T_{b(\hat{\tau}_{i,j})})$ and $S_i(T_{b(\tau_{b,k})})=S_i(T_{b(\hat{\tau}_{i,j})})$ for some $j$ on $\{\tau_{b,k}<\infty\}$ since $\tau_{b,k}<\infty$ means that the agent has played the $i^{th}$ action at the beginning of the cycle containing $\tau_{b,k}$. Then by \eqref{e45}
	\begin{equation*}
		\Prob(\theta_i(\tau_{b,k}+1)\leq\frac{\mu_1+\mu_i}{2}|\mathcal{\hat{H}}_{\tau_{b,k}})=Q\Bigg(\frac{\sqrt{1+M_i(T_{b(\hat{\tau}_{i,j})})}}{\sigma}\Big(\frac{S_i(T_{b(\hat{\tau}_{i,j})})}{1+M_i(T_{b(\hat{\tau}_{i,j})})}-\frac{\mu_1+\mu_i}{2}\Big)
		\Bigg)
	\end{equation*}
	on $\{\tau_{b,k}<\infty\}$ for some $j$.  However, for each $k$ that satisfies the condition $\{\tau_{b,k}<\infty, M_i(T_{b(\tau_{b,k})})< 32\sigma^2\frac{\log(T)}{\Delta_i^2}, M_1(\tau_{b,k})\leq T, A_{\tau_{e,k}}\neq i\}$, $[\tau_{b,k},\tau_{e,k}]$ has to be in a distinct cycle from the first $\lceil32\alpha\sigma^2\frac{\log(T)}{\Delta_i^2}\rceil$ cycles of the $i^{th}$ arm. Note that the distinctiveness follows from the fact that $A_{\tau_{e,k}}\neq i$ condition ends the cycle, while the upper bound on the number of cycles is the result of $\{M_i(T_{b(\tau_{b,k})})< 32\sigma^2\frac{\log(T)}{\Delta_i^2}\}$ and the way Algorithm \ref{algo} is implemented. These arguments naturally lead to
	\begin{align}
		&\Ex[\sum_{k=1}^{\infty}\1(\tau_{b,k}<\infty, M_i(T_{b(\tau_{b,k})})< 32\sigma^2\frac{\log(T)}{\Delta_i^2}, M_1(\tau_{b,k})\leq T, A_{\tau_{e,k}}\neq i)\frac{1}{\Prob( \theta_i(\tau_{b,k}+1)\leq\frac{\mu_1+\mu_i}{2}|\mathcal{\hat{H}}_{\tau_{b,k}})}]\nonumber\\
		&\leq \Ex\Bigg[\sum_{j=1}^{\lceil32\alpha\sigma^2\frac{\log(T)}{\Delta_i^2}\rceil}\1(\hat{\tau}_{i,j}<\infty)\frac{1}{Q\Bigg(\frac{\sqrt{1+M_i(T_{b(\hat{\tau}_{i,j})})}}{\sigma}\Big(\frac{S_i(T_{b(\hat{\tau}_{i,j})})}{1+M_i(T_{b(\hat{\tau}_{i,j})})}-\frac{\mu_1+\mu_i}{2}\Big)
			\Bigg)}\Bigg]\nonumber\\
		&\leq C(1+32\alpha\sigma^2\frac{\log(T)}{\Delta_i^2})\label{e51}
	\end{align}
	where the last inequality follows from Proposition \ref{prop2} and the range of $Q$ being from zero to one. In view of \eqref{e31} and \eqref{e35}, combining \eqref{e50} and \eqref{e51} shows that
	\begin{align}
		&\Ex[\sum_{t=1}^{T}\1(A_t=i, t\neq C_{e,k}\text{ for all }k,\theta_i(t)>\frac{\mu_1+\mu_i}{2}, M_i(T_{b(t)})< 32\sigma^2\frac{\log(T)}{\Delta_i^2})]\nonumber\\
		&\leq C(3+1088\alpha\sigma^2\frac{\log(T)}{\Delta_i^2})\label{e52}
	\end{align}
	where $C$ is an absolute constant. This finishes the analysis of the summand in \eqref{p3}.
	
	Finally, we will bound the summand in \eqref{p6}. However, most of the proof ideas will follow from earlier analysis. First note that if $\theta_1(t)>\frac{\mu_1+\mu_i}{2}$, while $\theta_j(t)\leq\frac{\mu_1+\mu_i}{2}$ for $j\geq 2$, then $A_t=1$:
	\begin{equation}
		\Prob(A_t=1|\Hp_{t-1})\geq\Prob(\theta_1(t)>\frac{\mu_1+\mu_i}{2}|\Hp_{t-1})\Prob(\theta_j(t)\leq\frac{\mu_1+\mu_i}{2}\text{ for all } j\neq 1|\Hp_{t-1})\label{e53}
	\end{equation} 
	where we also used the conditional independence of $\theta_j(t)$s given $\Hp_{t-1}$. On the other hand, if $A_t=i$ and $\theta_i(t)\leq \frac{\mu_1+\mu_i}{2}$, then $\theta_j(t)\leq \frac{\mu_1+\mu_i}{2}$ for all $j\geq 2$:
	\begin{equation}
		\Prob(A_t=i, \theta_i(t)\leq \frac{\mu_1+\mu_i}{2}|\Hp_{t-1})\leq\Prob(\theta_j(t)\leq\frac{\mu_1+\mu_i}{2}\text{ for all } j\neq 1|\Hp_{t-1})\label{e54}.
	\end{equation}
	The combination of \eqref{e53} and \eqref{e54} lead to
	\begin{equation}
		\Prob(A_t=i, \theta_i(t)\leq \frac{\mu_1+\mu_i}{2}|\Hp_{t-1})\leq\frac{\Prob(A_t=1|\Hp_{t-1})}{\Prob(\theta_1(t)>\frac{\mu_1+\mu_i}{2}|\Hp_{t-1})}.\label{e55}
	\end{equation}
	Note that considering the action selection process of Algorithm \ref{algo} where conditioned on the past observations $\theta_1$ has a Gaussian distribution, $\Prob(\theta_1(t)>\frac{\mu_1+\mu_i}{2}|\Hp_{t-1})$ will almost surely be non-zero.
	Then we have
	\begin{align}
		&\Ex[\1(A_t=i, A_{t-1}\neq1, \theta_i(t)\leq\frac{\mu_1+\mu_i}{2}, M_1(T_{b(t)})< 32\sigma^2\frac{\log(T)}{\Delta_i^2})]\nonumber\\
		&=\Ex[\Ex[\1(A_t=i, A_{t-1}\neq1, \theta_i(t)\leq\frac{\mu_1+\mu_i}{2}, M_1(T_{b(t)})< 32\sigma^2\frac{\log(T)}{\Delta_i^2})|\Hp_{t-1}]]\nonumber\\
		&=\Ex[\Prob(A_t=i, \theta_i(t)\leq \frac{\mu_1+\mu_i}{2}|\Hp_{t-1})\1(A_{t-1}\neq1, M_1(T_{b(t)})< 32\sigma^2\frac{\log(T)}{\Delta_i^2})]\label{e56}\\
		&\leq \Ex[\frac{\Prob(A_t=1|\Hp_{t-1})}{\Prob(\theta_1(t)>\frac{\mu_1+\mu_i}{2}|\Hp_{t-1})}\1(A_{t-1}\neq1, M_1(T_{b(t)})< 32\sigma^2\frac{\log(T)}{\Delta_i^2})]\label{e57}\\
		&=\Ex[\frac{1}{\Prob(\theta_1(t)>\frac{\mu_1+\mu_i}{2}|\Hp_{t-1})}\1(A_t=1,A_{t-1}\neq1, M_1(T_{b(t)})< 32\sigma^2\frac{\log(T)}{\Delta_i^2})]\label{e58},
	\end{align}
	where \eqref{e56} follows from moving terms measurable with respect to $\Hp_{t-1}$ out of the conditional expectation. The bound in \eqref{e55} leads to \eqref{e57}. Finally, using \eqref{e58} in \eqref{p6} shows  that
	\begin{align}
		&\Ex[\sum_{t=2}^{T}\1(A_t=i, A_{t-1}\neq1, \theta_i(t)\leq\frac{\mu_1+\mu_i}{2}, M_1(T_{b(t)})< 32\sigma^2\frac{\log(T)}{\Delta_i^2})]\nonumber\\
		&\leq \Ex[\sum_{t=2}^{T}\frac{1}{\Prob(\theta_1(t)>\frac{\mu_1+\mu_i}{2}|\Hp_{t-1})}\1(A_t=1,A_{t-1}\neq1, M_1(T_{b(t)})< 32\sigma^2\frac{\log(T)}{\Delta_i^2})]\label{e59}.
	\end{align}
	Note here that if $A_t=1$ and $A_{t-1}\neq 1$, then $t$ is either a cycle beginning or a cycle end point, which means that $t=\hat{\tau}_{1,j}$ for some $j$:
		\begin{align}
		\Prob(\theta_1(t)>\frac{\mu_1+\mu_i}{2}|\Hp_{t-1})&=Q\Bigg(\frac{\sqrt{1+M_1(T_{b(t)})}}{\sigma}\Big(\frac{\mu_1+\mu_i}{2}-\frac{S_1(T_{b(t)})}{1+M_1(T_{b(t)})}\Big)
		\Bigg)\label{e60}\\
		&=Q\Bigg(\frac{\sqrt{1+M_1(T_{b(\hat{\tau}_{1,j})})}}{\sigma}\Big(\frac{\mu_1+\mu_i}{2}-\frac{S_1(T_{b(\hat{\tau}_{1,j})})}{1+M_1(T_{b(\hat{\tau}_{1,j})})}\Big)
		\Bigg)\label{e61}
	\end{align}
	where the fact that conditioned on $\Hp_{t-1}$, $\theta_1(t)\sim\mathcal{N}(\frac{S_1(T_{b(t)})}{1+M_1(T_{b(t)})},\frac{\sigma^2}{1+M_1(T_{b(t)})})$ leads to \eqref{e60}. In addition, each $t$ that satisfies $\{A_t=1,A_{t-1}\neq1, M_1(T_{b(t)})< 32\sigma^2\frac{\log(T)}{\Delta_i^2}\}$ condition has to belong to a different cycle and the index $j$ can not be bigger than $\lceil32\alpha\sigma^2\frac{\log(T)}{\Delta_i^2}\rceil$. So, in view of \eqref{e59}, \eqref{e61} leads to
	\begin{align}
		&\Ex[\sum_{t=2}^{T}\1(A_t=i, A_{t-1}\neq1, \theta_i(t)\leq\frac{\mu_1+\mu_i}{2}, M_1(T_{b(t)})< 32\sigma^2\frac{\log(T)}{\Delta_i^2})]\nonumber\\
		&\leq\Ex\Bigg[\sum_{j=1}^{\lceil32\alpha\sigma^2\frac{\log(T)}{\Delta_i^2}\rceil}\1(\hat{\tau}_{1,j}<\infty)\frac{1}{Q\Bigg(\frac{\sqrt{1+M_1(T_{b(\hat{\tau}_{1,j})})}}{\sigma}\Big(\frac{\mu_1+\mu_i}{2}-\frac{S_1(T_{b(\hat{\tau}_{1,j})})}{1+M_1(T_{b(\hat{\tau}_{1,j})})}\Big)
			\Bigg)}\Bigg]\nonumber\\
		&\leq C(1+32\alpha\sigma^2\frac{\log(T)}{\Delta_i^2})\label{e62}.
	\end{align}
	Here the last inequality is the application of Proposition \ref{prop2}. However, this result finishes the proof of \eqref{r3} since the collection of bounds, \eqref{e23}, \eqref{e24}, \eqref{e25}, \eqref{e26}, \eqref{e52}, \eqref{e62}, prove that
	\begin{equation*}
		\Ex[N_i(T)]\leq 6+64\alpha\sigma^2\frac{\log(T)}{\Delta_i^2}+C(4+1120\alpha\sigma^2\frac{\log(T)}{\Delta_i^2}).
	\end{equation*}

\end{document}